\documentclass[12pt,draftcls,onecolumn]{IEEEtran}

\usepackage{amsmath}
\usepackage{amssymb}  
\usepackage{amsfonts}
\usepackage{mathtools}
\usepackage{bbm}
\usepackage{bm}
\usepackage{color}
\usepackage{graphics} 
\usepackage{graphicx}
\usepackage{epsfig} 
\usepackage{epstopdf}
\usepackage[noadjust]{cite}
\usepackage{tikz}
\usetikzlibrary{calc,positioning,shapes,shadows,arrows,fit}
\usepackage{fancyhdr}
\usepackage{fancyref}
\usepackage{hyperref}
\usepackage{float}

\usepackage{amsthm}

\newtheorem{theorem}{Theorem}
\newtheorem{lemma}{Lemma}

\theoremstyle{definition}

\newtheorem{definition}{Definition}
\newtheorem{remark}{Remark}
\newtheorem{assumption}{Assumption}

\newtheorem{problem}{Problem}

\begin{document}

\title{A Nonlinear Model Predictive Control Scheme for Cooperative Manipulation with Singularity and Collision Avoidance}

\author{Alexandros Nikou, Christos Verginis, Shahab Heshmati-alamdari and Dimos V. Dimarogonas
\thanks{Alexandros Nikou, Christos Verginis and Dimos V. Dimarogonas are with the ACCESS Linnaeus Center, School of Electrical Engineering, KTH Royal Institute of Technology, SE-100 44, Stockholm, Sweden and with the KTH Center for Autonomous Systems. Email: {\tt\small \{anikou, cverginis, dimos\}@kth.se}. Shahab  Heshmati-alamdari is with the Control Systems Lab, Department of Mechanical Engineering, National Technical University of Athens, 9 Heroon Polytechniou Street, Zografou 15780, Athens, Greece. Email: {\tt\small \{shahab\}@mail.ntua.gr}. This work was supported by the H2020 ERC Starting Grant BUCOPHSYS, the EU H2020 AEROWORKS project, the EU H2020 Co4Robots project, the Swedish Foundation for Strategic Research (SSF), the Swedish Research Council (VR) and the Knut och Alice Wallenberg Foundation (KAW).}}



\maketitle

\begin{abstract}
This paper addresses the problem of cooperative transportation of an object rigidly grasped by $N$ robotic agents. In particular, we propose a Nonlinear Model Predictive Control (NMPC) scheme that guarantees the navigation of the object to a desired pose in a bounded workspace with obstacles, while complying with certain input saturations of the agents. Moreover, the proposed methodology ensures that the agents do not collide with each other or with the workspace obstacles as well as that they do not pass through singular configurations. The feasibility and convergence analysis of the NMPC are explicitly provided. Finally, simulation results illustrate the validity and efficiency of the proposed method. 
\end{abstract}


%
\IEEEpeerreviewmaketitle

\section{Introduction}

Over the last years, multi-agent systems have gained a significant amount of attention, due to the advantages they offer with respect to single-agent setups. In the case of robotic manipulation and object transportation, difficult tasks involving heavy payloads as well as challenging maneuvers necessitate the employment of multiple robots. Fig. 1 depicts a system of two robotic mobile manipulators (KUKA youBots), each comprising of a moving base and a robotic arm of 5 Degrees of Freedom (DOF). 

Early works related to cooperative manipulation develop control architectures where the robotic agents communicate and share information with each other as well as completely decentralized schemes, where each agent uses only local information or observers, avoiding potential communication delays \cite{schneider1992object,liu1996decentralized,liu1998decentralized,zribi1992adaptive,khatib1996decentralized,caccavale2000task,gudino2004control}. Impedance and force/motion control constitutes the most common methodology used in the related literature \cite{schneider1992object,caccavale2008six,heck2013internal,erhart2013adaptive,erhart2013impedance,szewczyk2002planning,tsiamis2015cooperative,ficuciello2014cartesian,ponce2016cooperative,gueaieb2007robust}. However, most of the aforementioned works employ force/torque sensors to acquire knowledge of the manipulator-object contact forces/torques, which, however, may result to performance decline due to sensor noise or mounting difficulties. Recent technological advances allow to manipulator grippers to grasp rigidly certain objects (see e.g., \cite{grasping2014}), which, as shown in this work, can render the use of force/torque sensors unnecessary. 

Furthermore, in manipulation tasks, such as pose/force or trajectory tracking, collision with obstacles of the environment has been dealt with only by exploiting the extra degrees of freedom that appear in over-actuated robotic agents. Potential field-based algorithms may suffer from local minima and navigation functions \cite{koditschek1990robot} cannot be extended to multi-agent second order dynamical systems in a trivial way. Moreover, these methods usually result in high control input values near obstacles that need to be avoided, which might conflict the saturation of the actual motor inputs.

\begin{figure}
	\centering
	\includegraphics[width = 0.4\textwidth]{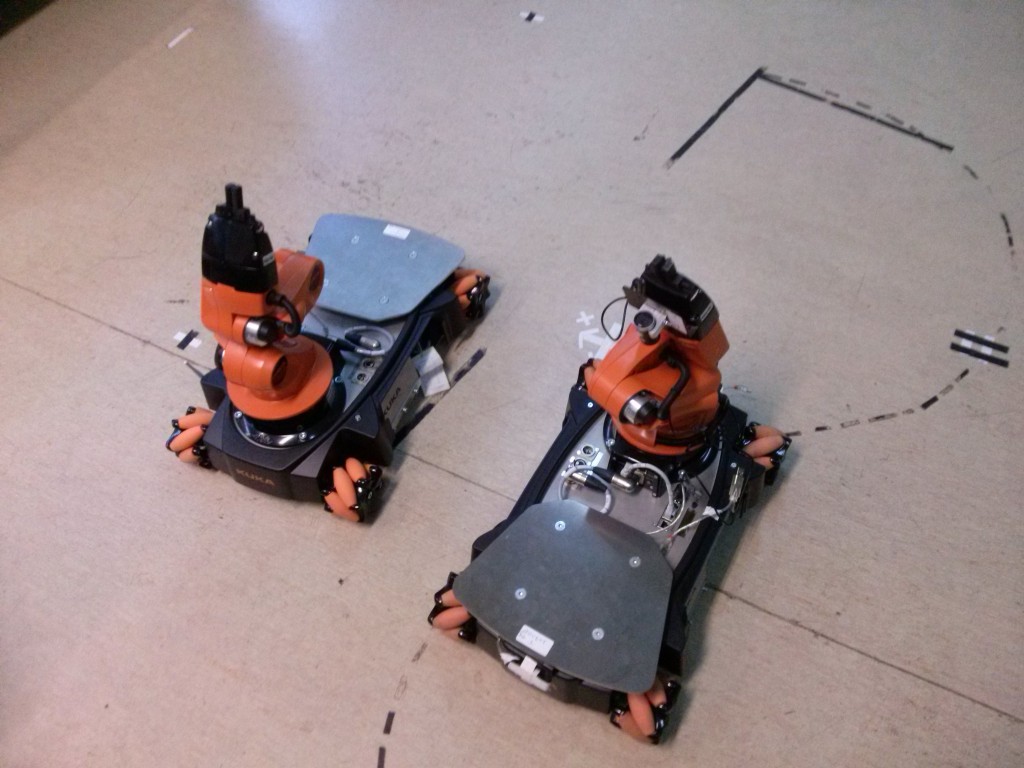}
	
	\caption{Two ground vehicles (KUKA youBots) consisting of a moving base and a attached manipulator with $5$ DOF.\label{fig:youbots}}
\end{figure} 

Another important property that concerns robotic manipulators is the singularities of the Jacobian matrix, which maps the joint velocities of the agent to a $6$D vector of generalized velocities. Such \textit{singular} \textit{kinematic} configurations, that indicate directions towards which the agent cannot move, must be always avoided, especially when dealing with task-space control in the end-effector \cite{siciliano}. In the same vein, \textit{representation} singularities can also occur in the mapping from coordinate rates to angular velocities of a rigid body. 

In this work, we aim to address the problem of cooperative manipulation of an object in a bounded workspace with obstacles. In particular, given $N$ agents that rigidly grasp an object, we design control inputs for the navigation of the object to a final pose, while avoiding inter-agent collisions as well as collisions with obstacles. Moreover, we take into account constraints that emanate from control input saturation as well kinematic and representation singularities.

For the design of a stabilizing feedback control law for each robot, such that the desired specifications are met, while satisfying constraints on the controls and the states, one would ideally look for a closed loop solution for the feedback law satisfying the constraints while optimizing the performance. However, typically the optimal feedback law cannot be found analytically, even in the unconstrained case, since it involves the solution of the corresponding Hamilton-Jacobi-Bellman partial differential equations. One approach to circumvent this problem is the repeated solution of an open-loop optimal control problem for a given state. The first part of the resulting open-loop input signal is implemented and the whole process is repeated. Control approaches using this strategy are referred to as Nonlinear Model Predictive Control (NMPC) (see e.g. \cite{morrari_npmpc, frank_2003_nmpc_bible, frank_1998_quasi_infinite, frank_2003_towards_sampled-data-nmpc, fontes_2001_nmpc_stability, grune_2011_nonlinear_mpc, camacho_2007_nmpc, cannon_2001_nmpc, borrelli_2013_nmpc, fontes_2007_modified_barbalat, alex_med_2017, alex_cdc_2017_timed_abstractions, ECC_2017_arxiv_dnmpc, alex_IJC_2017, alex_chris_ECC_2018}) which we aim to use in this work for the problem of the constraint cooperative manipulation of an object which is rigidly grasped by $N$ agents. To the best of the authors' knowledge, this problem has not been addressed in the related literature.

The remainder of the paper is structured as follows. Section \ref{sec:preliminaries} provides preliminary background. The system dynamics and the formal problem statement are given in Section \ref{sec:Problem-Formulation}. Section \ref{sec:solution} discusses the technical details of the solution and Section \ref{sec:simulation_results} is devoted to a simulation example. Finally, conclusions and future work are discussed in Section \ref{sec:conclusions}.

\section{Notation and Preliminaries} \label{sec:preliminaries}

The set of positive integers is denoted as $\mathbb{N}$ and the real $n$-coordinate space, with $n\in\mathbb{N}$, as $\mathbb{R}^n$;
$\mathbb{R}^n_{\geq 0}$ and $\mathbb{R}^n_{> 0}$ are the sets of real $n$-vectors with all elements nonnegative and positive, respectively. The notation $\mathbb{R}^{n\times n}_{\geq 0}$ and $\mathbb{R}^{n\times n}_{> 0}$, with $n\in\mathbb{N}$, stands for positive semi-definite and positive definite matrices, respectively. Moreover, $\lVert x \rVert$ is the Euclidean norm of a vector $x\in\mathbb{R}^n$. Given a set $S$, we denote by $|S|$ its cardinality and by $S^N = S \times \dots \times S$ its $N$-fold Cartesian product. Given the sets $S_1, S_2$, the \emph{set difference} and the \emph{Minkowski addition} are denoted by $\backslash, \oplus$, respectively, and are defined by $S_1 \backslash S_2 = \{s: s\in S_1 \ \text{and} \ s_2 \notin S_2\}$ and $S_1 \oplus S_2 = \{s_1 + s_2 : s_1 \in S_1, s_2 \in S_2\}$, respectively. The $n\times n$ identity matrix and the $n\times m$ matrix with zero entries, are denoted by $I_n$, $0_{n\times m}$ and $\mathbbm{1}_{n}$, respectively, with $n,m \in \mathbb{N}$. The largest singular value of matrix $A \in \mathbb{R}^{n \times m}$ is denoted as $\sigma_{\max}(A)$.

The vector connecting the origins of coordinate frames $\{A\}$ and $\{B$\} expressed in frame $\{C\}$ coordinates in $3$-D space is denoted as $p^{\scriptscriptstyle C}_{\scriptscriptstyle B/A} = [x_{\scriptscriptstyle B/A}, y_{\scriptscriptstyle B/A}, z_{\scriptscriptstyle B/A}]^\top \in\mathbb{R}^3$. Given $a\in\mathbb{R}^3, S(a)$ is the skew-symmetric matrix defined according to $S(a)b = a\times b$. We further denote as $\eta_{\scriptscriptstyle A/B} = [\phi_{\scriptscriptstyle A/B}, \theta_{\scriptscriptstyle A/B}, \psi_{\scriptscriptstyle A/B}]^\top\in\mathbb{T}^3\subseteq \mathbb{R}^3$ the $x$-$y$-$z$ Euler angles representing the orientation of frame $\{A\}$ with respect to frame $\{B\}$, with $\phi_{\scriptscriptstyle A/B}, \psi_{\scriptscriptstyle A/B}\in[-\pi,\pi]$ and $\theta_{\scriptscriptstyle A/B}\in[-\tfrac{\pi}{2}, \tfrac{\pi}{2}]$, where $\mathbb{T}^3$ is the $3$-D torus; Moreover, $R^{\scriptscriptstyle B}_{\scriptscriptstyle A}\in SO(3)$ is the rotation matrix associated with the same orientation and $SO(3)$ is the $3$-D rotation group.
The angular velocity of frame $\{B\}$ with respect to $\{A\}$, expressed in frame $\{C\}$ coordinates, is
denoted as $\omega^{\scriptscriptstyle C}_{{\scriptscriptstyle B/A}}\in \mathbb{R}^{3}$ and it holds that $\dot{R}^{\scriptscriptstyle B}_{\scriptscriptstyle A} = S(\omega^{\scriptscriptstyle A}_{\scriptscriptstyle B/A})R^{\scriptscriptstyle B}_{\scriptscriptstyle A}$. We further define the sets $\mathbb{M} = \mathbb{R}^3\times\mathbb{T}^3$, $\mathcal{N} = \{1,\dots,N\}$. We define also the set
\begin{align*} \label{eq:elipsoid_eq}
	\mathcal{O}_z &\triangleq \mathcal{O}(c_z, \beta_{1, z},\beta_{2, z}, \beta_{3, z}) \notag \\
	&= \left\{ p \in \mathbb{R}^3 : (p-c_z)^\top P (p-c_z) \le 1 \right\},
\end{align*}
as the set of an \emph{ellipsoid} in 3D, where $c_z \in \mathbb{R}^3$ is the center of the ellipsoid, $\beta_{1, z}, \beta_{2, z}, \beta_{3, z} \in \mathbb{R}_{> 0}$ the lengths of its three semi-axes and $z \ge 1$ is an index term. The eigenvector of matrix $P \in \mathbb{R}^3$ define the principal axes of the ellipsoid, and the eigenvalues of $P$ are: $\beta_{1, z}^{-2}, \beta_{2, z}^{-2}$ and $\beta_{3, z}^{-2}$.  For notational brevity, when a coordinate frame corresponds to an inertial frame of reference $\left\{I\right\}$, we will omit its explicit notation (e.g., $p_{\scriptscriptstyle B} = p^{\scriptscriptstyle I}_{\scriptscriptstyle B/I}, \omega_{\scriptscriptstyle B} = \omega^{\scriptscriptstyle I}_{\scriptscriptstyle B/I}, R_{\scriptscriptstyle A} = R^{\scriptscriptstyle I}_{\scriptscriptstyle A} $, etc.). Finally, all vector and matrix differentiations will be with respect to an inertial frame $\{I\}$, unless otherwise stated.

\begin{definition} \label{def:class_K}
	(\cite{khalil_nonlinear_systems}) A continuous function $f: [0, \alpha] \to \mathbb{R}_{\ge 0}, \alpha \in \mathbb{R}_{> 0}$ is said to belong to \emph{class} $\mathcal{K}$, if is strictly increasing and $f(0) = 0$.
\end{definition}

\begin{lemma} (\cite{michalska_1994_barbalat})
	Let $\gamma$ be a continuous, positive definite function and $x$ be an absolutely continuous function on $\mathbb{R}$. If the following holds:
	\begin{itemize}
		\item $\|x(\cdot)\| < \infty, \|\dot{x}(\cdot)\| < \infty$,
		\item $\displaystyle \lim_{t \to \infty} \int_{0}^{t} \gamma(x(s)) ds < \infty$.
	\end{itemize}
	Then, it holds that: $\lim_{t \to \infty} \|x(t)\| = 0$.
\end{lemma}

\section{Problem Formulation}
\label{sec:Problem-Formulation}

Consider a bounded and convex workspace $\mathcal{W} \subseteq \mathbb{R}^{3}$ consisting of $N$ robotic agents rigidly grasping an object, as shown in Fig. \ref{fig:Two-robotic-arms}, and $Z$ obstacles described by the ellipsoids $\mathcal{O}_z, z \in \mathcal{Z} = \{1, \dots, Z\}$. The free space is denoted as $\mathcal{W}_{\text{free}} = \mathcal{W}\backslash \bigcup_{z \in \mathcal{Z}} \mathcal{O}_z$. The agents are considered to be fully actuated and they consist of a base that is able to move around the workspace (e.g., mobile or aerial vehicle) and a robotic arm.
The reference frames corresponding to the $i$-th end-effector and the object's center of
mass are denoted with $\left\{ E_{i}\right\} $ and $\left\{ O\right\} $,
respectively, whereas $\left\{ I\right\} $ corresponds to an inertial
reference frame. The rigidity of the grasps implies that the agents can exert any forces/torques along every direction to the object. 
We consider that each agent $i$ knows the position and velocity only of its own state as well as its own and the object's geometric parameters. Moreover, no interaction force/torque measurements or on-line communication is required.

\begin{figure}
	\centering
	\includegraphics[width = 0.4\textwidth]{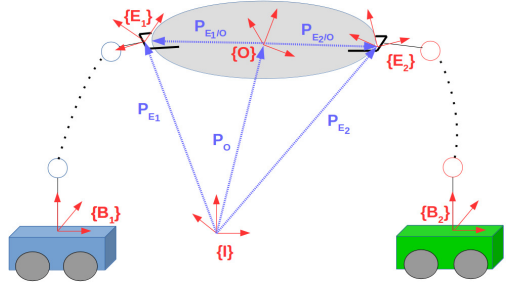}
	
	\caption{Two robotic arms rigidly grasping an object with the corresponding frames.\label{fig:Two-robotic-arms}}
\end{figure} 

\subsection{System model}
\label{subsec:system_model}

\subsubsection{Robotic Agents} \label{subsubsec: agent dynamics}

We denote by $q_i:\mathbb{R}_{\geq 0}\to\mathbb{R}^{n_i}$ the joint space variables of agent $i\in\mathcal{N}$, with $n_i = n_{\alpha_i} + 6$, $q_i(t) = [p^\top_{\scriptscriptstyle B_i}(t),\eta^\top_{\scriptscriptstyle B_i}(t), \alpha^\top_i(t)]^\top$, where $p_{\scriptscriptstyle B_i} = [x_{\scriptscriptstyle B_i}, y_{\scriptscriptstyle B_i}, z_{\scriptscriptstyle B_i}]^\top:\mathbb{R}_{\geq 0} \to\mathbb{R}^{3},\eta_{\scriptscriptstyle B_i} = [\phi_{\scriptscriptstyle B_i}, \theta_{\scriptscriptstyle B_i}, \psi_{\scriptscriptstyle B_i}]^\top:\mathbb{R}_{\geq 0}\to\mathbb{T}^{3}\subseteq \mathbb{R}^{3}$ is the position and Euler-angle orientation of the agent's base, and $\alpha_i:\mathbb{R}_{\geq 0}\to\mathbb{R}^{n_{\alpha_i}},n_{\alpha_i}>0$, are the degrees of freedom of the robotic arm. The overall joint space configuration vector is denoted as $q = [q^\top_1,\dots,q^\top_N]^\top\in\mathbb{R}^n$, with $n = \sum_{i\in\mathcal{N}}n_i$.
In addition, we denote as $p_{\scriptscriptstyle E_i}:\mathbb{R}^{n_i}\to\mathbb{R}^3,\eta_{\scriptscriptstyle E_i}:\mathbb{R}^{n_i}\to\mathbb{T}^3\subseteq \mathbb{R}^3$ the position and Euler-angle orientation of agent $i$'s end-effector. Let also $v_i:\mathbb{R}^{n_i}\times\mathbb{R}^{n_i}\rightarrow\mathbb{R}^{6}$ denote the velocity of agent $i$'s end-effector, with $v_i(q_i,\dot{q}_i) = [\dot{p}^\top_{\scriptscriptstyle E_i}, \omega^\top_{\scriptscriptstyle E_i}]^\top$, whereas $\dot{p}_{\scriptscriptstyle B_i}, \omega_{\scriptscriptstyle B_i}:\mathbb{R}^{n_i}\times\mathbb{R}^{n_i}\to\mathbb{R}^{3}$ are the linear and angular velocity, respectively, of the agent's base.

We consider that each agent $i\in\mathcal{N}$ has access to its own state $q_i$ as well as $\dot{p}^{\scriptscriptstyle B_i}_{\scriptscriptstyle B_i}, \omega^{\scriptscriptstyle B_i}_{\scriptscriptstyle B_i}$, and $\dot{\alpha}_i$ via on-board sensors. Then, $\dot{p}_{\scriptscriptstyle B_i}, \omega_{\scriptscriptstyle B_i}$ can be obtained via $\dot{p}_{\scriptscriptstyle B_i} = R_{B_i}(\eta_{B_i})\dot{p}^{\scriptscriptstyle B_i}_{\scriptscriptstyle B_i}$, $\omega_{\scriptscriptstyle B_i} = R_{\scriptscriptstyle B_i}(\eta_{\scriptscriptstyle B_i})\omega^{\scriptscriptstyle B_i}_{\scriptscriptstyle B_i}$, where $R_{\scriptscriptstyle B_i}:\mathbb{T}^3\to SO(3)$ is the rotation matrix of the agent $i$'s base.
Moreover, $\dot{\eta}_{\scriptscriptstyle B_i}$ is related to $\omega_{\scriptscriptstyle B_i}$ via $\omega_{\scriptscriptstyle B_i} = J_{\scriptscriptstyle B_i}(\eta_{\scriptscriptstyle B_i})\dot{\eta}_{\scriptscriptstyle B_i}$, where $J_{\scriptscriptstyle B_i}:\mathbb{T}^3\to\mathbb{R}^{3\times3}$, with 
\begin{equation}
J_{\scriptscriptstyle B_i}(\eta_{\scriptscriptstyle B_i}) = \begin{bmatrix}
1 & 0 & \sin(\theta_{\scriptscriptstyle B_i}) \\
0 & \cos(\phi_{\scriptscriptstyle B_i}) & -\cos(\theta_{\scriptscriptstyle B_i})\sin(\phi_{\scriptscriptstyle B_i}) \\
0 & \sin(\phi_{\scriptscriptstyle B_i}) & \cos(\theta_{\scriptscriptstyle B_i})\cos(\phi_{\scriptscriptstyle B_i})
\end{bmatrix}. \notag
\end{equation}
The pose of the $i$th end-effector can be computed via 
\begin{align}
p_{\scriptscriptstyle E_i}(q_i) &= p_{\scriptscriptstyle B_i} + R_{\scriptscriptstyle B_i}(\eta_{\scriptscriptstyle B_i})k_{p_i}(\alpha_i), \notag \\
\eta_{\scriptscriptstyle E_i}(q_i) &= k_{\eta_i}(\eta_{\scriptscriptstyle B_i},\alpha_i) \notag,
\end{align} 
where $k_{p_i}:\mathbb{R}^{n_{\alpha_i}}\to\mathbb{R}^3, k_{\eta_i}:\mathbb{T}^3\times\mathbb{R}^{n_{\alpha_i}}\to\mathbb{T}^3$ are the forward kinematics of the robotic arm \cite{siciliano}. Then, $v_i$ can be computed as 
\begin{align}
v_i(q_i,\dot{q}_i) &= 
\begin{bmatrix}
\dot{p}_{\scriptscriptstyle E_i}(q_i,\dot{q}_i) \\ 
\omega_{\scriptscriptstyle E_i}(q_i,\dot{q}_i)
\end{bmatrix} \notag \\
&= 
\begin{bmatrix}
\dot{p}_{\scriptscriptstyle B_i} - S(R_{\scriptscriptstyle B_i}k_{p_i})\omega_{\scriptscriptstyle B_i} + R_{\scriptscriptstyle B_i}\tfrac{\partial k_{p_i}}{\partial \alpha_i}\\ 
\omega_{\scriptscriptstyle B_i} + R_{\scriptscriptstyle B_i}J_{\scriptscriptstyle A_i}\dot{\alpha}_i \label{eq:diff_kinematics}
\end{bmatrix},
\end{align}
where $J_{\scriptscriptstyle A_i}:\mathbb{R}^{n_{\alpha_i}}\to\mathbb{R}^{3\times n_{\alpha_i}}$ is the angular Jacobian of the robotic arm with respect to the agent's base. The differential kinematics \eqref{eq:diff_kinematics} can be written as 
\begin{equation}
v_i(q_i,\dot{q}_i) = \begin{bmatrix}
\dot{p}_{\scriptscriptstyle E_i}(q_i,\dot{q}_i) \\ \omega_{\scriptscriptstyle E_i}(q_i,\dot{q}_i)
\end{bmatrix} = J_i(q_i)\dot{q}_i,	\label{eq:diff_kinematics_2}
\end{equation}
where $J_i:\mathbb{R}^{n_i}\to\mathbb{R}^{6\times n_i}$ is the agent Jacobian matrix, with 
\begin{align*}
J_i(q_i) = \begin{bmatrix}
I_3 & -S(R_{\scriptscriptstyle B_i}(\eta_{\scriptscriptstyle B_i})k_{p_i}(\alpha_i))J_{\scriptscriptstyle B_i}(\eta_{\scriptscriptstyle B_i}) & R_{\scriptscriptstyle B_i}(\eta_{\scriptscriptstyle B_i})\tfrac{\partial k_{p_i}(\alpha_i)}{\partial \alpha_i} \\
0_{3\times 3} & J_{\scriptscriptstyle B_i}(\eta_{\scriptscriptstyle B_i}) & R_{\scriptscriptstyle B_i}(\eta_{\scriptscriptstyle B_i})J_{\scriptscriptstyle A_i}(q_i)
\end{bmatrix}.
\end{align*}

\begin{remark}
	Note that $J_{\scriptscriptstyle B_i}$ becomes singular at representation singularities, when $\theta_{\scriptscriptstyle B_i} = \pm \tfrac{\pi}{2}$ and $J_i$ becomes singular at kinematic singularities defined by the set
	\begin{equation*} \label{eq:Q}
	\mathcal{Q}_i = \{q_i\in\mathbb{R}^{n_i} : \det(J^\top_iJ_i) = 0 \}, i \in \mathcal{N}. \notag
	\end{equation*}
	In the following, we will aim at guaranteeing that $q_i$ will always be in the closed set:
	\begin{equation*} \label{eq:Q_tilde}
	\widetilde{\mathcal{Q}}_i = \{q_i\in\mathbb{R}^{n_i} : \lvert\det(J^\top_iJ_i)\rvert \geq \varepsilon > 0\}, i \in \mathcal{N}, \notag
	\end{equation*}
	for a small positive constant $\varepsilon$. 
\end{remark}

The joint-space dynamics for agent $i\in\mathcal{N}$ can be computed using the Lagrangian formulation:
\begin{equation}
B_{i}(q_i)\ddot{q}_i+N_{i}(q_i,\dot{q}_i)\dot{q}_i+g_{q_i}(q_i)  = \tau_{i} - J^\top_i\lambda_{i},   \label{eq:manipulator joint space dynamics}
\end{equation}
where $B_{i}:\mathbb{R}^{n_i}\rightarrow\mathbb{R}^{n_i\times n_i}$ is the joint-space positive definite inertia matrix,  $N_i:\mathbb{R}^{n_i}\times\mathbb{R}^{n_i}\rightarrow\mathbb{R}^{n_i\times n_i}$ represents the joint-space Coriolis matrix, $g_{q_i}:\mathbb{R}^{n_i}\rightarrow\mathbb{R}^{n_i}$
is the joint-space gravity vector, $\lambda_{i}\in\mathbb{R}^{6}$ is the generalized force vector that agent $i$ exerts on the object and $\tau_i\in\mathbb{R}^{n_i}$ is the vector of generalized joint-space inputs, with $\tau_i = [\lambda^\top_{\scriptscriptstyle B_i}, \tau^\top_{\alpha_i}]^\top$, where $\lambda_{\scriptscriptstyle B_i} = [f^\top_{\scriptscriptstyle B_i}, \mu^\top_{\scriptscriptstyle B_i}]^\top \in \mathbb{R}^6$ is the generalized force vector on the center of mass of the agent's base and $\tau_{\alpha_i}\in\mathbb{R}^{n_{\alpha_i}}$ is the torque inputs of the robotic arms' joints.
By inverting \eqref{eq:manipulator joint space dynamics} and using \eqref{eq:diff_kinematics_2} and its derivative, we can obtain the task-space agent dynamics \cite{siciliano}:
\begin{equation}
M_{i}(q_i)\dot{v}_i+C_{i}(q_i,\dot{q}_i)v_i+g_{i}(q_i) 
= u_{i} - \lambda_{i},   \label{eq:manipulator task space dynamics}
\end{equation}
with the corresponding task-space terms:
\begin{align}
M_i(q_i) &= \left[J_i(q_i)B^{-1}_i(q_i)J^\top_i(q_i)\right]^{-1}, \notag \\
C_i(q_i,\dot{q}_i)J_i(q_i)\dot{q}_i &= M_i(q_i)\left[J_i(q_i)B^{-1}_i(q_i)N_i - \dot{J}_i(q_i)\right]\dot{q}_i, \notag \\
g_i(q_i) &=  M_i(q_i)J_i(q_i)B^{-1}_i(q_i)g_{q_i}(q_i). \notag 
\end{align}

The task-space input wrench $u_i$ can be translated to the joint space inputs $\tau_i\in\mathbb{R}^{{n}_i}$ via $\tau_{i}=J_{i}^{\top}(q_i)u_{i}+(I_{{ {n_{i}}}}-J^{\top}_{i}(q_i)\bar{J}^{\top}_{i}(q_i))\tau_{i_{0}}$, where  $\bar{J}_i$ is a generalized inverse of $J_i$ \cite{siciliano}. The term $\tau_{i_{0}}$ concerns over-actuated agents and does not contribute to end-effector forces.

We define by $\mathcal{A}_i(q_i) \triangleq \mathcal{O}_i, i \in \mathcal{N}$,
the ellipsoid that bounds the $i$ th agent's volume with the corresponding centers $c_i$ and semi-axes $\beta_{i, 1}, \beta_{i, 2}, \beta_{i, 3}$, i.e., the workspace of the arm of agent $i$ \cite{siciliano} enlarged so that it includes the $i$th base. Note that $\mathcal{A}_i$ depends on $q_i$ and can be explicitly found.

\subsubsection{Object Dynamics}  \label{subsubsec:object dynamics}

Regarding the object, we denote as $x_{\scriptscriptstyle O}:\mathbb{R}_{\geq 0}\rightarrow\mathbb{M}$, $v_{\scriptscriptstyle O}:\mathbb{R}_{\geq 0}$ $\rightarrow\mathbb{R}^6$ the pose and velocity of the object's center of mass, with $x_{\scriptscriptstyle O}(t) = [p^\top_{\scriptscriptstyle O}(t), \eta^\top_{\scriptscriptstyle O}(t)]^\top$, $p_{\scriptscriptstyle O}(t) = [x_{\scriptscriptstyle O}(t), y_{\scriptscriptstyle O}(t), z_{\scriptscriptstyle O}(t)]^\top$, $\eta_{\scriptscriptstyle O}(t) = [\phi_{\scriptscriptstyle O}(t), \theta_{\scriptscriptstyle O}(t), \psi_{\scriptscriptstyle O}(t)]^\top$ and $v_{\scriptscriptstyle O}(t) = [\dot{p}^\top_{\scriptscriptstyle O}(t), \omega^\top_{\scriptscriptstyle O}(t)]^\top$. The second order dynamics of the object are given by:
\begin{subequations} \label{eq:object dynamics} 
	\begin{align}
	\dot{x}_{\scriptscriptstyle O}(t) & = J^{-1}_{\scriptscriptstyle O_r}(x_{\scriptscriptstyle O})v_{\scriptscriptstyle O}(t), \label{eq:object dynamics_1} \\ 
	\lambda_{\scriptscriptstyle O} & = M_{\scriptscriptstyle O}(x_{\scriptscriptstyle O})\dot{v}_{{\scriptscriptstyle O}}(t)+C_{{\scriptscriptstyle O}}(x_{\scriptscriptstyle O},v_{\scriptscriptstyle O})v_{{\scriptscriptstyle O}}(t)+g_{\scriptscriptstyle O}(x_{\scriptscriptstyle O}),\label{eq:object dynamics_2} 
	\end{align}
\end{subequations}
where $M_{\scriptscriptstyle O}:\mathbb{M}\rightarrow\mathbb{R}^{6\times6}$ is the positive definite inertia matrix, $C_{{\scriptscriptstyle O}}:\mathbb{M}\times\mathbb{R}^6\rightarrow\mathbb{R}^{6\times6}$ is the Coriolis matrix, $g_{\scriptscriptstyle O}:\mathbb{M}\rightarrow\mathbb{R}^{6}$ is the gravity vector, which are derived from the Newton-Euler formulation. In addition, $J_{\scriptscriptstyle O_r}:\mathbb{M}\rightarrow\mathbb{R}^{6\times6}$ is the object representation Jacobian $J_{\scriptscriptstyle O_r}(x_{\scriptscriptstyle O}) = \text{diag}\{I_3, J_{\scriptscriptstyle O_{r,\theta}}(x_{\scriptscriptstyle O}) \}$, with
\begin{equation}
J_{\scriptscriptstyle O_{r,\theta}}(x_{\scriptscriptstyle O}) = \begin{bmatrix}
1 & 0 & \sin(\theta_{\scriptscriptstyle O}) \\
0 & \cos(\phi_{\scriptscriptstyle O}) & -\cos(\theta_{\scriptscriptstyle O})\sin(\phi_{\scriptscriptstyle O}) \\
0 & \sin(\phi_{\scriptscriptstyle O}) & \cos(\theta_{\scriptscriptstyle O})\cos(\phi_{\scriptscriptstyle O})
\end{bmatrix}, \notag
\end{equation}
which is singular when $\theta_{\scriptscriptstyle O}= \pm \tfrac{\pi}{2}$. Finally, $\lambda_{\scriptscriptstyle O}\in\mathbb{R}^6$ is the force vector acting on the object's center of mass. Also, similarly to the robotic agents, we define by $\mathcal{C}_{\scriptscriptstyle O}(x_{\scriptscriptstyle O}) \triangleq \mathcal{O}_{\scriptscriptstyle O}$,
as the bounding ellipsoid of the object.

\subsubsection{Coupled Dynamics} \label{subsubsec: coupled dynamics}

Consider $N$ robotic agents rigidly grasping an object. Then, the coupled system object-agents behaves like a closed-chain robot and we can express the object's pose and velocity as a function of $q_i$ and $\dot{q}_i$, $\forall i\in\mathcal{N}$. 
In view of Fig. \ref{fig:Two-robotic-arms}, we have that 
\begin{subequations} \label{eq:coupled kinematics}
	\begin{align}
	p_{{\scriptscriptstyle E_{i}}}(q_i(t)) &=  p_{{\scriptscriptstyle O}}(t)+ p_{{\scriptscriptstyle E_{i}/O}}(q_i) \notag\\
	&=p_{{\scriptscriptstyle O}}(t)+R_{{\scriptscriptstyle E_i}}(t) p^{\scriptscriptstyle E_i}_{{\scriptscriptstyle E_{i}/O}} \label{eq:coupled kinematics_p},\\ 
	\eta_{\scriptscriptstyle E_i}(q_i(t)) &=  \eta_{\scriptscriptstyle O}(t) +  \eta_{\scriptscriptstyle E_i/O}, \label{eq:coupled kinematics_ksi} 	
	\end{align}
\end{subequations}
$\forall i\in\mathcal{N}$, where $p^{\scriptscriptstyle E_i}_{{\scriptscriptstyle E_{i}/O}}$ represents the constant distance and  $\eta_{\scriptscriptstyle E_i/O}$ the relative orientation offset between the $i$th agent's end-effector and the object's center of mass, which are considered known. 
The grasp rigidity implies that $\omega_{\scriptscriptstyle E_i}$ $=\omega_{\scriptscriptstyle O}$, $\forall i\in\mathcal{N}$. Therefore, by differentiating \eqref{eq:coupled kinematics_p}, we obtain 
\begin{equation}
v_i(q_i,\dot{q}_i(t))=J_{{\scriptscriptstyle O_i}}(q_i)v_{{\scriptscriptstyle O}}(t),\label{eq:object-end-effector jacobian}
\end{equation}
which, by time differentiation, yields
\begin{equation}
\dot{v}_i(t) = J_{{\scriptscriptstyle O_i}}(q_i)\dot{v}_{{\scriptscriptstyle O}}(t) + \dot{J}_{{\scriptscriptstyle O_i}}(q_i) v_{{\scriptscriptstyle O}}(t),\label{eq:object-end-effector jacobian_dot}
\end{equation}
where $J_{\scriptscriptstyle O_i}:\mathbb{R}^{n}\rightarrow\mathbb{R}^{6\times6}$ is a smooth mapping representing the Jacobian from the object to the $i$-th agent:  
\begin{equation*}
J_{{\scriptscriptstyle O_i}}(q_i)=\left[\begin{array}{cc}
I_3 & S(p_{{\scriptscriptstyle O/E_{i}}}(q_i))\\
0_{{\scriptscriptstyle 3\times3}} & I_3
\end{array}\right],
\label{eq:jacobian O_i}
\end{equation*}
and is always full rank due to the grasp rigidity.


\begin{remark}
	Since the geometric object parameters $p^{\scriptscriptstyle E_i}_{\scriptscriptstyle E_i/O}$ and $\eta_{\scriptscriptstyle E_i/O}$ are known, each agent can compute $p_{\scriptscriptstyle O},\eta_{\scriptscriptstyle O}$ and $v_{\scriptscriptstyle O}$ simply by inverting \eqref{eq:coupled kinematics} and \eqref{eq:object-end-effector jacobian}, respectively, without employing any sensory data. In the same vein, all agents can also compute the object's bounding ellipsoid $\mathcal{C}_{\scriptscriptstyle O}$, which depends on $q$.
\end{remark}

The Kineto-statics duality \cite{siciliano} along with the grasp rigidity suggest that the 
force $\lambda_{\scriptscriptstyle O}$ acting on the object center of mass and the generalized forces $\lambda_i, i\in\mathcal{N}$, exerted by the
agents at the contact points are related through 
\begin{equation}
\lambda_{\scriptscriptstyle O}=G^{\top}(q)\bar{\lambda},\label{eq:grasp matrix}
\end{equation}
where $\bar{\lambda}=[\lambda^{\top}_{1}, \cdots, \lambda^{\top}_{N}]^{\top}\in\mathbb{R}^{6N}$ and $G:\mathbb{R}^{{n}}\rightarrow\mathbb{R}^{6N\times6}$ is the grasp matrix, with $G(q)=[J_{{\scriptscriptstyle O_1}}^{\top},\cdots,J_{{\scriptscriptstyle O_N}}^{\top}]^{\top}$. 

Next, we substitute \eqref{eq:object-end-effector jacobian} and \eqref{eq:object-end-effector jacobian_dot} in \eqref{eq:manipulator task space dynamics} and we obtain in vector form after rearranging terms:
\begin{align}
&\bar{\lambda} = u - \bar{M}(q)G(q)\dot{v}_{\scriptscriptstyle O} - (\bar{M}(q)\dot{G}(q,\dot{q}) + \bar{C}(q ,\dot{q})G(q)) v_{\scriptscriptstyle O} -  \bar{g}(q), \label{eq:coupled dynamics 1}
\end{align}
where we have used the stack forms $\bar{M}=  \text{diag}\{\left[M_{i}\right]_{i\in\mathcal{N}}\}$, $\bar{C} = \text{diag}\{\left[C_{i}\right]_{i\in\mathcal{N}}\}$, $\bar{g}=[g_{1}^{\top}, \dots, g_{N}^{\top}]^{\top}$, and $u=[u_{1}^{\top}, \dots, u_{N}^{\top}]^{\top}$.
By substituting \eqref{eq:coupled dynamics 1} and \eqref{eq:object dynamics} in \eqref{eq:grasp matrix} and by noticing from \eqref{eq:coupled kinematics} that $x_{\scriptscriptstyle O}$ depends on $q$ owing to the grasp rigidity, we obtain the coupled dynamics:
\begin{equation}
\widetilde{M}(q)\dot{v}_{{\scriptscriptstyle O}}+\widetilde{C}(q,\dot{q})v_{{\scriptscriptstyle O}}+\widetilde{g}(q) =G^{\top}(q)u,\label{eq:coupled dynamics 2}
\end{equation}
where:
\begin{subequations}
\begin{align}
\widetilde{M}(q) &=  M_{\scriptscriptstyle O}(q)+G^{\top}(q)\bar{M}(q)G(q), \label{eq:coupled terms_M}\\
\widetilde{C}(q,\dot{q})  &=  C_{\scriptscriptstyle O}(q)+G^{\top}(q)\bar{M}(q)\dot{G}(q,\dot{q})+G^{\top}(q)\bar{C}(q)G(q), \\
\widetilde{g}(q) &= g_{\scriptscriptstyle O}(q) + G^{\top}(q)\bar{g}(q),
\end{align}
\end{subequations}

\begin{remark}
	Note that the agents dynamics under consideration hold for generic robotic agents comprising of a moving base and a robotic arm. Hence, the considered framework can be applied for mobile, aerial, or underwater manipulators. 
\end{remark}

We can now formulate the problem considered in this work:
\begin{problem}
	Consider $N$ robotic agents rigidly grasping an object, governed by the coupled dynamics \eqref{eq:coupled dynamics 2}. Given the desired pose $x_{\scriptscriptstyle O,\text{des}}$, design the control input $u: \mathbb{R}_{\ge 0} \to \mathbb{R}^{6N}$ such that $\lim\limits_{t\to\infty}x_{\scriptscriptstyle O}(t) = x_{\scriptscriptstyle O,\text{des}}$, while ensuring the satisfaction of the following collision avoidance and singularity properties:
	\begin{enumerate}
		\item $\mathcal{A}_i(q_i)\cap\mathcal{O}_z = \emptyset, \forall i\in\mathcal{N}, z\in\mathcal{Z}$,
		\item $\mathcal{C}_{\scriptscriptstyle O}(x_{\scriptscriptstyle O})\cap\mathcal{O}_z = \emptyset, \forall z \in \mathcal{Z}$,
		\item $\mathcal{A}_i(q_i)\cap\mathcal{A}_{j}(q_{j}) = \emptyset, \forall i, j \in \mathcal{N}, i\neq j$,
		\item $-\tfrac{\pi}{2} < -\bar{\theta} \le  \theta_{\scriptscriptstyle O} \le \bar{\theta} < \tfrac{\pi}{2}$,
		\item $-\tfrac{\pi}{2} < -\bar{\theta} \le \theta_{\scriptscriptstyle B_i} \le \bar{\theta} < \tfrac{\pi}{2}$,
		\item $q_i \in \widetilde{\mathcal{Q}}_i$.
	\end{enumerate} 
	for a $0 < \bar{\theta} < \frac{\pi}{2}$, as well as the input and velocity magnitude and input constraints: $\lvert \tau_{i_k} \rvert \leq \bar{\tau}_i, \lvert \dot{q}_{i_k} \rvert \leq \bar{\dot{q}}_i, \forall k\in\{1,\dots,n_i\}, i\in\mathcal{N}$, for some positive constants $\bar{\tau}_i, \bar{\dot{q}}_i, i\in\mathcal{N}$.
	
	The aforementioned constraints correspond to the following specifications:
	\begin{itemize}
		\item $1)$ stands for collision avoidance between the agents and the obstacles.
		\item $2)$ stands for collision avoidance between the object and the obstacles.
		\item $3)$ stands for collision avoidance between the agents.
		\item $4)$ stands for representation singularity avoidance of the object.
		\item $5)$ stands for representation singularity avoidance of the agents' bases.
		\item $6)$ stands for kinematic singularity avoidance of the agents. 
	\end{itemize}
	
\end{problem}

In order to solve the aforementioned problem, we need the following reasonable assumption regarding the workspace:

\begin{assumption} \label{ass:feasility_assumption}
	(Problem Feasibility Assumption) The distance between any pair of obstacles is sufficiently large such that the coupled system object-agents can navigate among them without collisions.
\end{assumption}

\noindent We also define the following sets for every $i \in \mathcal{N}$:
\begin{align}
S_{i, {\scriptscriptstyle O}}(q) &= \{q_i\in\mathbb{R}^{n_i} : \mathcal{A}_i(q_i)\cap\mathcal{O}_z \neq \emptyset, \forall z \in \mathcal{Z} \}, \notag \\
S_{i, \scriptscriptstyle A}(q) &= \{q_i\in\mathbb{R}^{n_i} : \mathcal{A}_i(q_i)\cap\mathcal{A}_{j}(q_{j}) \neq \emptyset, \forall j \in \mathcal{N} \backslash \{i\} \}, \notag \\
S_{\scriptscriptstyle O}(x_{\scriptscriptstyle O}) &= \{x_{\scriptscriptstyle O} \in \mathbb{M}: \mathcal{C}_{\scriptscriptstyle O}(x_{\scriptscriptstyle O}) \cap \mathcal{O}_z \neq \emptyset  \}. \notag
\end{align}
associated with the desired collision-avoidance properties.

\section{Problem Solution} \label{sec:solution}

In this section, a systematic solution to Problem 1 is introduced. Our overall approach builds on designing a  Nonlinear Model Predictive control scheme the system of the manipulators and the object. Nonlinear Model Predictive Control (see e.g. \cite{morrari_npmpc, frank_2003_nmpc_bible, frank_1998_quasi_infinite, frank_2003_towards_sampled-data-nmpc, fontes_2001_nmpc_stability, grune_2011_nonlinear_mpc, camacho_2007_nmpc, cannon_2001_nmpc, borrelli_2013_nmpc}) have been proven suitable for dealing with nonlinearities and state and input constraints.

The coupled agents-object \emph{nonlinear dynamics} can be written in compact form as follows:

\begin{equation} \label{eq:main_system}
\dot{x} = f(x,u)= \begin{bmatrix}
f_1(x,u) \\
f_2(x,u) \\
f_3(x,u) 
\end{bmatrix}, x(0) = x_0,
\end{equation}
where $x = [x_{\scriptscriptstyle O}^\top, v_{\scriptscriptstyle O}^\top, q^\top]^\top \in \mathbb{R}^{n+12}, u \in \mathbb{R}^{6N}$ and 
\begin{align}
f_1(x,u) &= J^{-1}_{\scriptscriptstyle O_r}(x_{\scriptscriptstyle O})v_{\scriptscriptstyle O}, \notag \\
f_2(x,u) &= \widetilde{M}^{-1}(q)\left[G^\top(q)u - \widetilde{C}(q,\dot{q})v_{\scriptscriptstyle O} - \widetilde{g}(q) \right],  \notag \\
f_3(x,u) &= \hat{J}(q)J_{\scriptscriptstyle O}(q)\widetilde{I}v_{\scriptscriptstyle O}, \notag 
\end{align}
where we have also used that:
\begin{align}
\hat{J}(q) &= \text{diag}\left\{ \left[ (J^\top_iJ_i)^{-1}J^\top_i \right]_{i\in\mathcal{N}}  \right\} \in\mathbb{R}^{n\times 6N}, \notag \\
J_{\scriptscriptstyle O}(q) &= \text{diag}\left\{ \left[ J_{\scriptscriptstyle O_i} \right]_{i\in\mathcal{N}}  \right\} \in\mathbb{R}^{6N\times 6N}, \notag \\
\widetilde{I} &= 
\begin{bmatrix}
I_6, \cdots, I_6
\end{bmatrix}^\top \in \mathbb{R}^{6N \times 6}.
\end{align}
The expression for $f_3(x,u)$ is derived by employing \eqref{eq:object-end-effector jacobian_dot} and \eqref{eq:diff_kinematics_2}. Note that $f$ is \emph{locally Lipschitz continuous} in its domain since it is continuously differentiable in its domain. Next, we define the respective errors:
\begin{align}
e(t) &= x(t) - x_{\text{des}} =
\begin{bmatrix}
x_{\scriptscriptstyle O}(t) \\ 
v_{\scriptscriptstyle O}(t) \\
q(t)
\end{bmatrix}
-
\begin{bmatrix}
x_{\scriptscriptstyle O,\text{des}} \\
\dot{x}_{\scriptscriptstyle O,\text{des}} \\
q_{\text{des}} 
\end{bmatrix} =
\begin{bmatrix}
x_{\scriptscriptstyle O}(t)-x_{\scriptscriptstyle O,\text{des}} \\ 
v_{\scriptscriptstyle O}(t) \\
q(t) - q_{\text{des}} 
\end{bmatrix} \in\mathbb{R}^{n+12}, \label{eq:error}
\end{align}
where $q_{\text{des}}=[q_{1,\text{des}},\dots,q_{N,\text{des}}]^\top$ is appropriately chosen such that $x_{\scriptscriptstyle O}(t) = x_{\scriptscriptstyle O, \text{des}}, \forall t \text{ s.t. } q(t)  = q_{\text{des}}$ (see \eqref{eq:coupled kinematics}), and $\dot{x}_{\scriptscriptstyle O,\text{des}} = \dot{q}_{\text{des}} = 0$. The error dynamics are then $\dot{e}(t) = f(x(t),u(t))$, which can be appropriately transformed to be written as:
\begin{equation}
\dot{e}(t) = f_e(e(t),u(t)), e(0) = e_0 = x(0)-x_{\text{des}}. \label{eq:error_dynamics}
\end{equation}
where $f_e(t) \triangleq f(e(t)+x_{\text{des}}, u(t)$. By ignoring over-actuated input terms, we have that $\tau_i = J_i^\top(q_i)u_i$, which becomes
\begin{align}
\lVert \tau_i \rVert \leq \bar{\tau}_i \Leftrightarrow \sigma_{\min,i}\lVert u_i \rVert \leq \bar{\tau}_i,
\end{align} 
where we have employed the property $\sigma_{\min}(J^\top_i)\lVert u_i\rVert \leq \lVert J^\top_i u_i \rVert$, with $\sigma_{\min}(J^\top_i)$ denoting the minimum singular value of $J^\top_i$, which is strictly positive, if the constraint $q_i\in\widetilde{\mathcal{Q}}_i$ is always satisfied. Hence, the constraint $\lvert \tau_{i_k} \rvert \leq \bar{\tau}_i$ is equivalent to 
\begin{equation}
\lVert u_i \rVert \leq \frac{\bar{\tau}_i}{\sigma_{\min}(J^\top_i)}, \forall i\in\mathcal{N}.
\end{equation}

Let us now define the following set $U \subseteq \mathbb{R}^{6N}$:
\begin{equation}
U = \left\{ u\in\mathbb{R}^{6N} :  \lVert u_i \rVert \leq \frac{\bar{\tau}_i}{\sigma_{\min}(J^\top_i)}, \forall i\in\mathcal{N} \right\},
\end{equation}
as the set that captures the control input constraints of the error dynamics system \eqref{eq:error_dynamics}. Define also the set $X \subseteq \mathbb{R}^{n+12}$:
\begin{align*}
X &= \Big\{x \in \mathbb{R}^{n+12} : \theta_{\scriptscriptstyle O}(t) \in [-\bar{\theta}, \bar{\theta}], \theta_{\scriptscriptstyle B_i}(t) \in [-\bar{\theta}, \bar{\theta}], \notag \\
&\hspace{30mm} \lvert \dot{q}_{k_i} \rvert \leq \bar{\dot{q}}_i, q_i \in\widetilde{\mathcal{Q}}_i\backslash\left(\mathcal{S}_{i, {\scriptscriptstyle O}}(q_i) \cup \mathcal{S}_{i, \scriptscriptstyle A}(q_i)\right), \notag \\
&\hspace{30mm} x_{\scriptscriptstyle O}\in\mathbb{R}^3\backslash S_{\scriptscriptstyle O}(x_{\scriptscriptstyle O}), \forall t \in \mathbb{R}_{\ge 0} \Big\}.
\end{align*}
The set $X$ captures all the state constraint of the system dynamics \eqref{eq:main_system}. In view of \eqref{eq:error}, we define the set $E \subseteq \mathbb{R}^{n+12}$ as:
\begin{equation*}
E = \{e \in \mathbb{R}^{n+12}: e \in X \oplus (-x_\text{des}) \},
\end{equation*}
as the set that captures all the constraints of the error dynamics system \eqref{eq:error_dynamics}.

The problem in hand is the design of a control input $u(t)\in U$ such that $\lim_{t \to \infty} \|e(t)\| = 0$ while ensuring  $e(t) \in E, \forall t\in\mathbb{R}_{\geq 0}$.
In order to solve the aforementioned problem, we propose a Nonlinear Model Predictive scheme, that is presented hereafter.

Consider a sequence of sampling times $\{t_i\}_{i \ge 0}$ with a constant sampling period $0 < h < T_p$, where is $T_p$ is the prediction horizon, such that: 
\begin{equation} \label{eq:t_i_equals_h}
t_{i+1} = t_i + h, \forall \ i \ge 0.
\end{equation}
In the sampling-data NMPC, a finite-horizon open-loop optimal control problem (OCP) is solved at discrete sampling time instants $t_i$ based on the current state error information $e(t_i)$. The solution is an optimal control signal $\hat{u}(t)$, for $t \in [t_i,t_i+T_p]$. For more details, the reader is referred to \cite{frank_2003_nmpc_bible}. The open-loop input signal applied in between the sampling instants is given by the solution of the following Optimal Control Problem (OCP):
\begin{subequations}
	\begin{align}
	&\hspace{-4mm}\min\limits_{\hat{u}(\cdot)} J(e(t_i),\hat{u}(\cdot)) \notag \\ 
	&\hspace{-4mm}= \min\limits_{\hat{u}(\cdot)} \left\{  V(\hat{e}(t_i+T_p)) + \int_{t_i}^{t_i+T_p} \Big[ F(\hat{e}(s), \hat{u}(s)) \Big] ds \right\}  \label{eq:mpc_minimazation} \\
	&\hspace{-4mm}\text{subject to:} \notag \\
	&\hspace{1mm} \dot{\hat{e}}(s) = f_e(\hat{e}(s),\hat{u}(s)), \hat{e}(t_i) = e(t_i), \label{eq:diff_mpc} \\
	&\hspace{1mm} \hat{e}(s) \in E, \hat{u}(s) \in U, s \in [t_i,t_i+T_p], \\
	&\hspace{1mm} \hat{e}(t_i+T_p)\in\mathcal{E}_f, \label{eq:mpc_terminal_set}
	\end{align}
\end{subequations}
where the hat $\hat{\cdot}$ denotes the predicted variables (internal to the controller), i.e. $\hat{e}(\cdot)$ is the solution of \eqref{eq:diff_mpc} driven by the control input $\hat{u}(\cdot): [t_i, t_i+T_p] \to \mathcal{U}$ with initial condition $e(t_i)$. Note that the predicted values are not necessarily the same with the actual closed-loop values (see \cite{frank_2003_nmpc_bible}). The term $F: E \times U \to \mathbb{R}_{\ge 0}$, is the \emph{running cost}, and is chosen as:
\begin{equation} \label{eq:terminal_cost}
F(e,u) = e^\top Q e + u^\top R u.
\end{equation}
The terms $V: E \to \mathbb{R}_{ > 0}$ and $\mathcal{E}_f$ are the \emph{terminal penalty cost} and \emph{terminal set}, respectively, and are used to enforce the stability of the system (see Section 4.2). The terminal cost is given by $V(e) = e^\top P e$.
The terms $Q \in \mathbb{R}^{(n+12)\times(n+12)}_{\geq 0}$, $P \in \mathbb{R}^{(n+12)\times(n+12)}_{>0}$ and $R \in \mathbb{R}^{6N\times6N}_{>0}$ are chosen as:
\begin{align*}
Q = \text{diag} \{ \widetilde{q}_1, \dots, \widetilde{q}_{n+12} \}, P = \text{diag} \{ \widetilde{p}_1, \dots, \widetilde{p}_{n+12} \}, R = \text{diag} \{ \widetilde{r}_1, \dots, \widetilde{r}_{6N} \}.
\end{align*}
where $\widetilde{q}_i \in\mathbb{R}_{\geq 0}, \widetilde{p}_i\in\mathbb{R}_{>0} , \forall i \in \{1, \dots, n+12 \} $ and $\widetilde{r}_j\in\mathbb{R}_{>0}, \forall j \in \{1, \dots, 6N\}$ are constant weights. 

\begin{lemma} \label{lemma:F_i_bounded_K_class}
There exist functions $\alpha_1$, $\alpha_2 \in \mathcal{K}_{\infty}$ such that: $$\alpha_1\big(\|z\|\big) \leq F\big(e, u\big) \leq \alpha_2\big(\|z \|\big),$$ for every $z \triangleq \left[ e^\top, u^\top\right]^\top \in \mathcal{E} \times \mathcal{U}$.
\end{lemma}
\begin{proof}
	The proof can be found in Appendix \ref{app:proof_lemma_1}.
\end{proof}

The solution of the OCP  \eqref{eq:mpc_minimazation}-\eqref{eq:mpc_terminal_set} at time $t_i$ provides an optimal control input denoted by $\hat{u}^\star(t; e(t_i))$, for $t \in [t_i, t_i+T_p]$. It defines the open-loop input that is applied to the system until the next sampling instant $t_{i+1}$:
\begin{equation} \label{eq:control_input_star}
u(t; e(t_i)) = \hat{u}^\star(t_i; e(t_i)), t \in [t_i, t_{i+1}).
\end{equation} 
The corresponding \emph{optimal value function} is given by:
\begin{equation} \label{eq:J_star}
J^\star(e(t_i)) \triangleq J^\star(e(t_i), \hat{u}^\star(\cdot; e(t_i))).
\end{equation}
where $J(\cdot)$ as is given in \eqref{eq:mpc_minimazation}. The control input $ u(t; e(t_i))$ is a feedback, since it is recalculated at each sampling instant using the new state information. The solution of \eqref{eq:error_dynamics} starting at time $t_1$ from an initial condition $e(t_1)$, applying a control input $u: [t_1, t_2] \to \mathcal{U}$ is denoted by $e(s; u(\cdot), e(t_1)), s \in [t_1, t_2]$. The predicted state of the system \eqref{eq:error_dynamics} at time $t_i+s, s > 0$ is denoted by $\hat{e}(t_i+s; u(\cdot), e(t_i))$ and 
it is based on the measurement of the state $e(t_i)$ at time $t_i$, when a control input $u(\cdot; e(t_i))$ is applied to the system \eqref{eq:error_dynamics} for the time period $[t_i, t_i+s]$. Thus, it holds that:
\begin{equation} \label{eq:predicted_state_relation}
e(t_i) = \hat{e}(t_i; u(\cdot), e(t_i)).
\end{equation} 

\noindent We define an admissible control input as:

\begin{definition} \label{def:admissible_control_input}
	A control input $u: [0, T_p] \to \mathbb{R}^{6N}$ for a state $e_0$ is called \emph{admissible}, if all the following hold:
	\begin{enumerate}
		\item $u(\cdot)$ is piecewise continuous;
		\item $u(s) \in U, \forall \ s \in [0, T_p]$;
		\item $e(s; u(\cdot), e_0) \in E, \forall \ s \in [0, T_p]$;
		\item $e(T_p; u(\cdot), e_0) \in \mathcal{E}_f$;
	\end{enumerate}
\end{definition}

\begin{lemma}
	The terminal penalty function $V(\cdot)$ is Lipschitz continues in $\mathcal{E}_f$, with Lipschitz constant $L_V = 2\varepsilon_0 \sigma_{\max}(P)$, for all $e(t) \in \mathcal{E}_f$.
\end{lemma}
\begin{proof}
	The proof can be found in Appendix \ref{app:lemma_2}.
\end{proof}

Through the following theorem, we guarantee the stability of the system which is the solution to Problem 1.

\begin{theorem}
	Consider the Assumptions 1,2. Suppose also that:
	\begin{enumerate}
		\item The OCP \eqref{eq:mpc_minimazation}-\eqref{eq:mpc_terminal_set} is feasible for the initial time $t = 0$.
		\item The terminal set $\mathcal{E}_f \subseteq E$ is closed, with $0_{n+12} \in \mathcal{E}_f$.
		\item The terminal set $\mathcal{E}_f$ is chosen such that there exists an admissible control input $u_f: [0, h] \to \mathcal{U}$ such that for all $e(s) \in \mathcal{E}_f$ it holds that:
		\begin{enumerate}
			\item $e(s) \in \mathcal{E}_f, \forall \ s \in [0, h]$.
			\item $\displaystyle \frac{\partial V}{\partial{e}} f_e(e(s), u_f(s)) + F(e(s), u_f(s)) \le 0, \forall \ s \in [0, h].$
		\end{enumerate}
	\end{enumerate}
	Then, the closed loop trajectories of the system \eqref{eq:error_dynamics}, converges to the set $\mathcal{E}_f$, as $t \to \infty$.
\end{theorem}

\begin{proof}
As usual in predictive control the proof consists of two parts: in the first part it is established that initial feasibility implies feasibility afterwards. Based on this result it is then shown that the error $e(t)$ converges to the terminal set $\mathcal{E}_f$. The feasibility analysis can be found in Appendix \ref{app:feasibility}. The convergence analysis can be found in \ref{app:convergence_analysis}.
\end{proof}

\begin{figure}[t!]
	\centering
	\includegraphics[width = 0.6\textwidth]{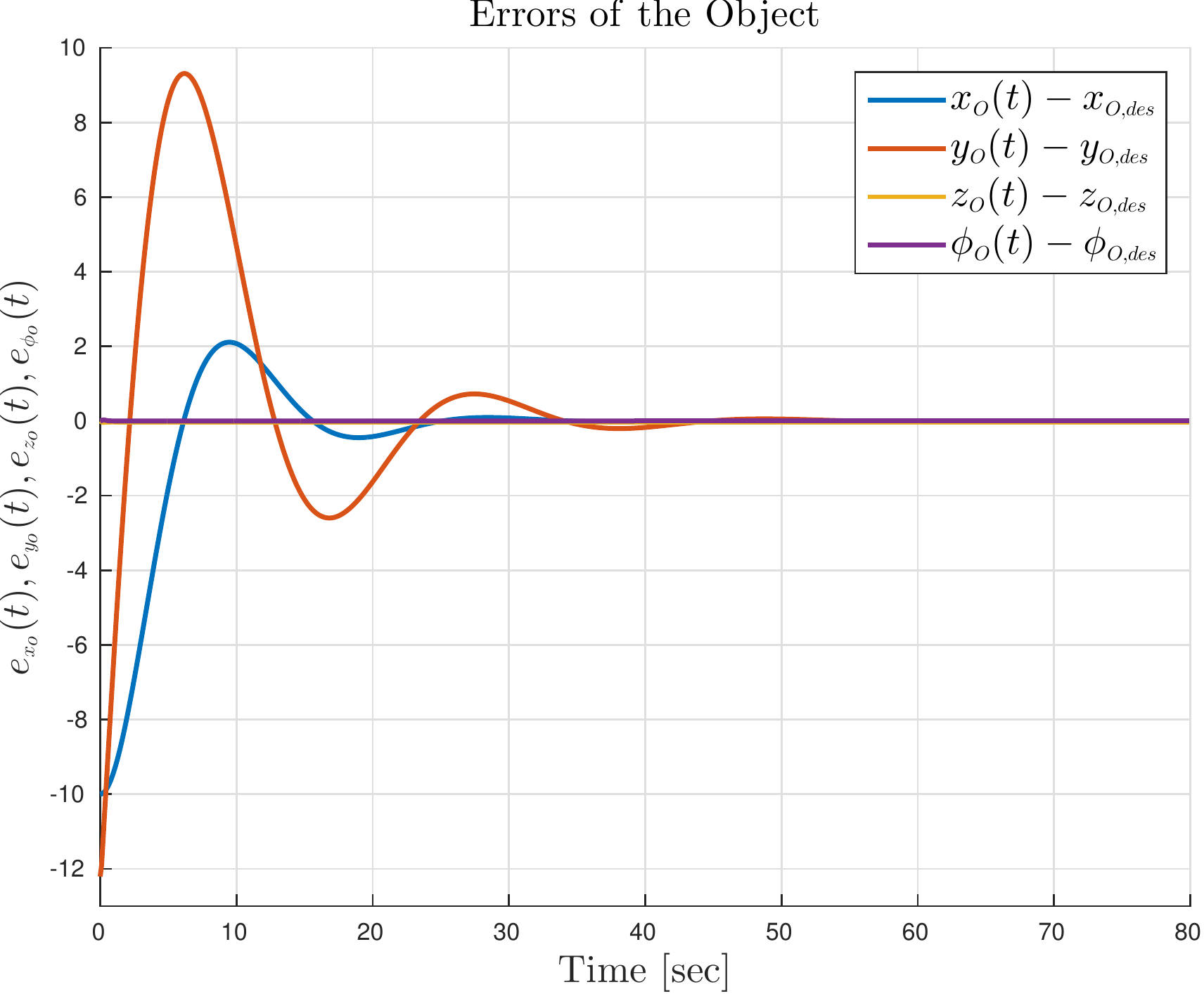}
	\caption{The errors of the object.\label{fig:sim1}}
\end{figure} 

\begin{figure}[t!]
	\centering
	\includegraphics[width = 0.6\textwidth]{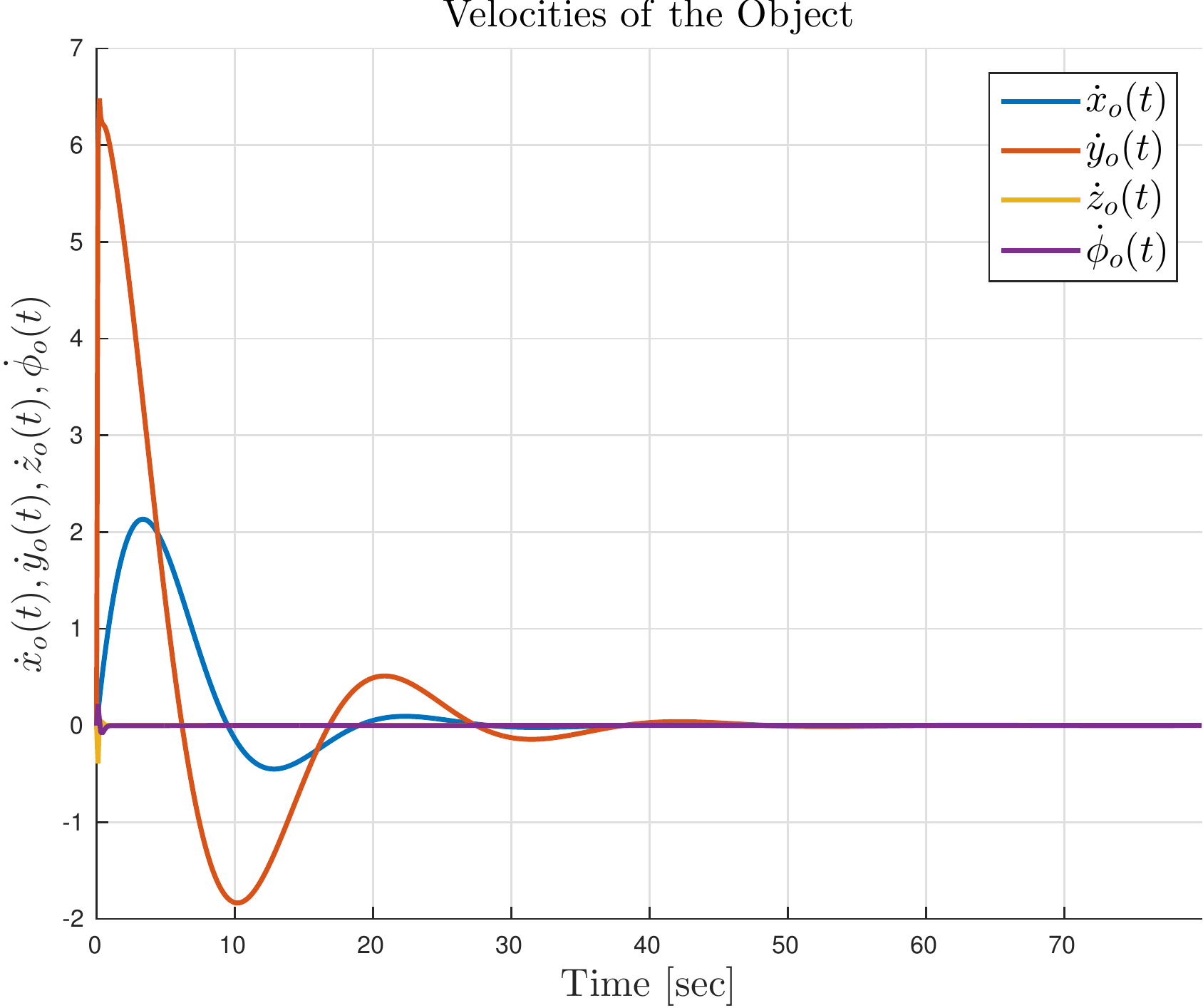}
	\caption{The velocities of the object.\label{fig:sim2}}
\end{figure} 

\section{Simulation Results} \label{sec:simulation_results}

To demonstrate the efficiency of the proposed control protocol, we consider two simulation scenarios.

\textbf{Scenario 1:} Consider $N=2$ ground vehicles equipped with $2$ DOF manipulators, rigidly grasping an object with $n_1 = n_2 = 4, n = n_1+n_2 = 8$. From \eqref{eq:main_system} we have that $x = [x_{\scriptscriptstyle O}^\top, v_{\scriptscriptstyle O}^\top, q^\top]^\top \in \mathbb{R}^{16}, u \in \mathbb{R}^{8}$, with $x_{\scriptscriptstyle O} = [p_{\scriptscriptstyle O}^\top, \phi_{\scriptscriptstyle O}]^\top \in \mathbb{R}^4$, $v_{\scriptscriptstyle O} = [\dot{p}_{\scriptscriptstyle O}^\top, \omega_{\scriptscriptstyle x_O}]^\top \in \mathbb{R}^4$, $p_{\scriptscriptstyle O} = [x_{\scriptscriptstyle O}, y_{\scriptscriptstyle O}, z_{\scriptscriptstyle O}]^\top \in \mathbb{R}^3, q = [q_1^\top, q_2^\top]^\top \in \mathbb{R}^8$, $q_i = [p_{\scriptscriptstyle B_i}^\top, \alpha_i^\top]^\top \in \mathbb{R}^4$, $p_{\scriptscriptstyle B_i} = [x_{\scriptscriptstyle B_i}, y_{\scriptscriptstyle B_i}]^\top \in \mathbb{R}^2$, $\alpha_i = [\alpha_{i_1}, \alpha_{i_2}]^\top \in \mathbb{R}^2, i \in \{1,2\}$. The manipulators become singular when $\sin(\alpha_{i_1}) = 0 \}, i \in \{1,2\}$, thus the state constraints for the manipulators are set to:
\begin{align}
\varepsilon < \alpha_{1_1} < \frac{\pi}{2}-\varepsilon, & -\frac{\pi}{2}+\varepsilon < \alpha_{1_2} < \frac{\pi}{2}-\varepsilon, \notag \\
-\frac{\pi}{2} + \varepsilon < \alpha_{2_1} < -\varepsilon,& -\frac{\pi}{2}+\varepsilon < \alpha_{2_2} < \frac{\pi}{2}-\varepsilon. \notag
\end{align}
We also consider the input constraints:
\begin{equation*}
-10 \le u_{i,j}(t) \le 10, i \in \{1, 2\}, j \in \{1,\dots,4\}.
\end{equation*}
The initial conditions are set to:
\begin{align*}
x_{\scriptscriptstyle O}(0) &= \left[0, -2.2071, 0.9071, \frac{\pi}{2} \right]^\top, v_{\scriptscriptstyle O}(0) = \left[0, 0, 0, 0\right]^\top, \\
q_{1}(0) &= \left[0, 0, \frac{\pi}{4}, \frac{\pi}{4}\right]^\top, q_{2}(0) = \left[0, -4.4142, -\frac{\pi}{4}, -\frac{\pi}{4}\right]^\top.
\end{align*}
The desired goal states are set to: 
\begin{align*}
x_{\scriptscriptstyle O, \text{des}} &= \left[10, 10, 0.9071, \frac{\pi}{2} \right]^\top, v_{\scriptscriptstyle O, \text{des}} = \left[0, 0, 0, 0\right]^\top, \\
q_{1, \text{des}} &= \left[10, 12.2071, \frac{\pi}{4}, \frac{\pi}{4}\right]^\top, q_{2, \text{des}} = \left[10, 7.7929, -\frac{\pi}{4}, -\frac{\pi}{4}\right]^\top.
\end{align*}  
We set an obstacle between the initial and the desired pose of the object. the obstacle is spherical with center $[5,5,1]$ and radius $2$. The sampling time is $h = 0.1 \sec$, the horizon is set to $T_p = 0.3 \sec$, and the total simulation time is $80 \sec$; The matrices $P, Q, R$ are set to:
\begin{equation}
P = Q = 10  I_{16 \times 16}, R = 2 I_{8 \times 8}. \notag 
\end{equation}
The simulation results are depicted in Fig. \ref{fig:sim1}- Fig. \ref{fig:sim8}, which shows that the states of the agents as well as the states of the object converge to the desired ones while guaranteeing that the obstacle is avoided and all state and input constraints are met.

\begin{figure}[t!]
	\centering
	\includegraphics[width = 0.6\textwidth]{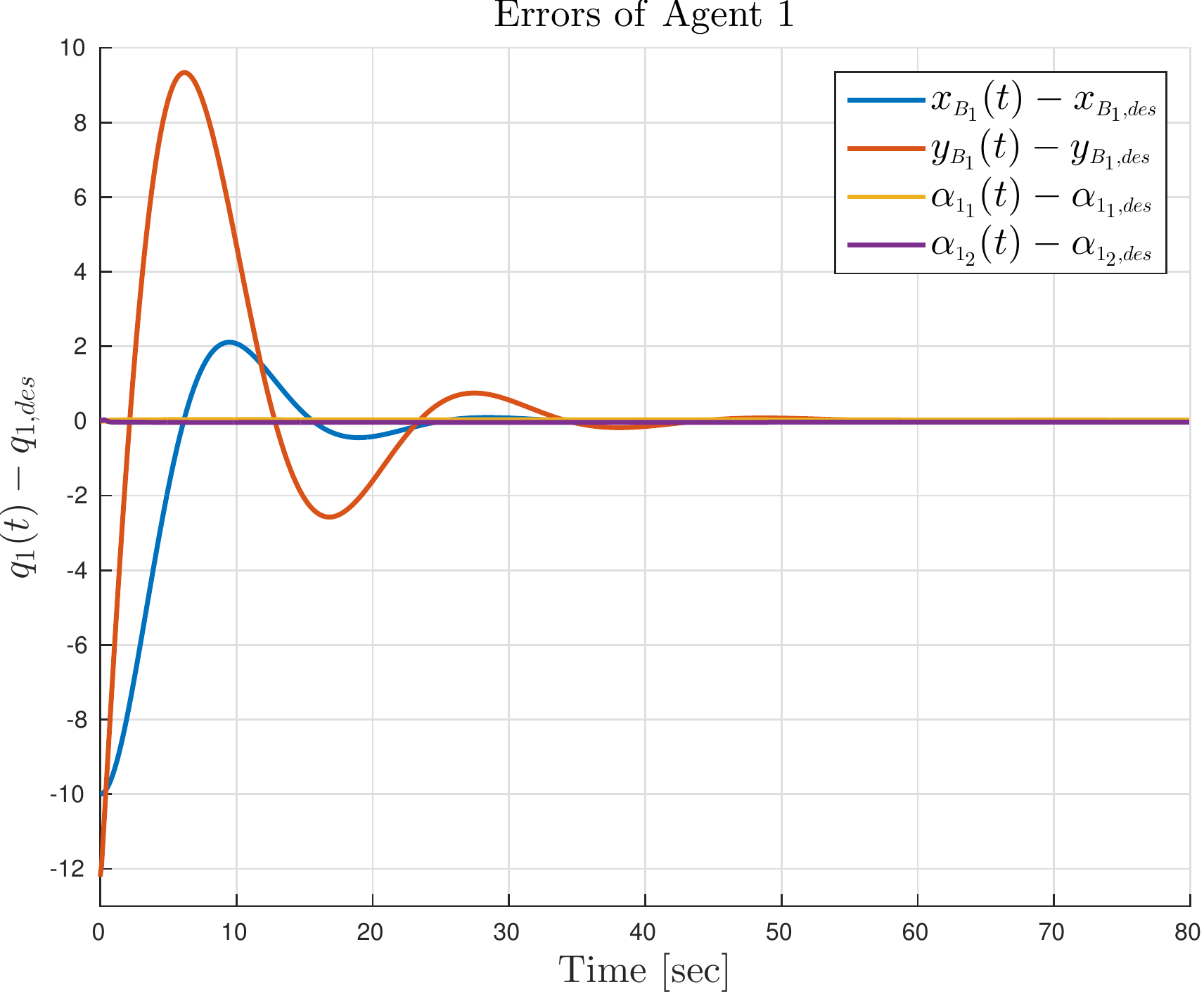}	
	\caption{The errors of vehicle $1$ as well as the errors of the manipulator.\label{fig:sim3}}
\end{figure} 

\begin{figure}[t!]
	\centering
	\includegraphics[width = 0.6\textwidth]{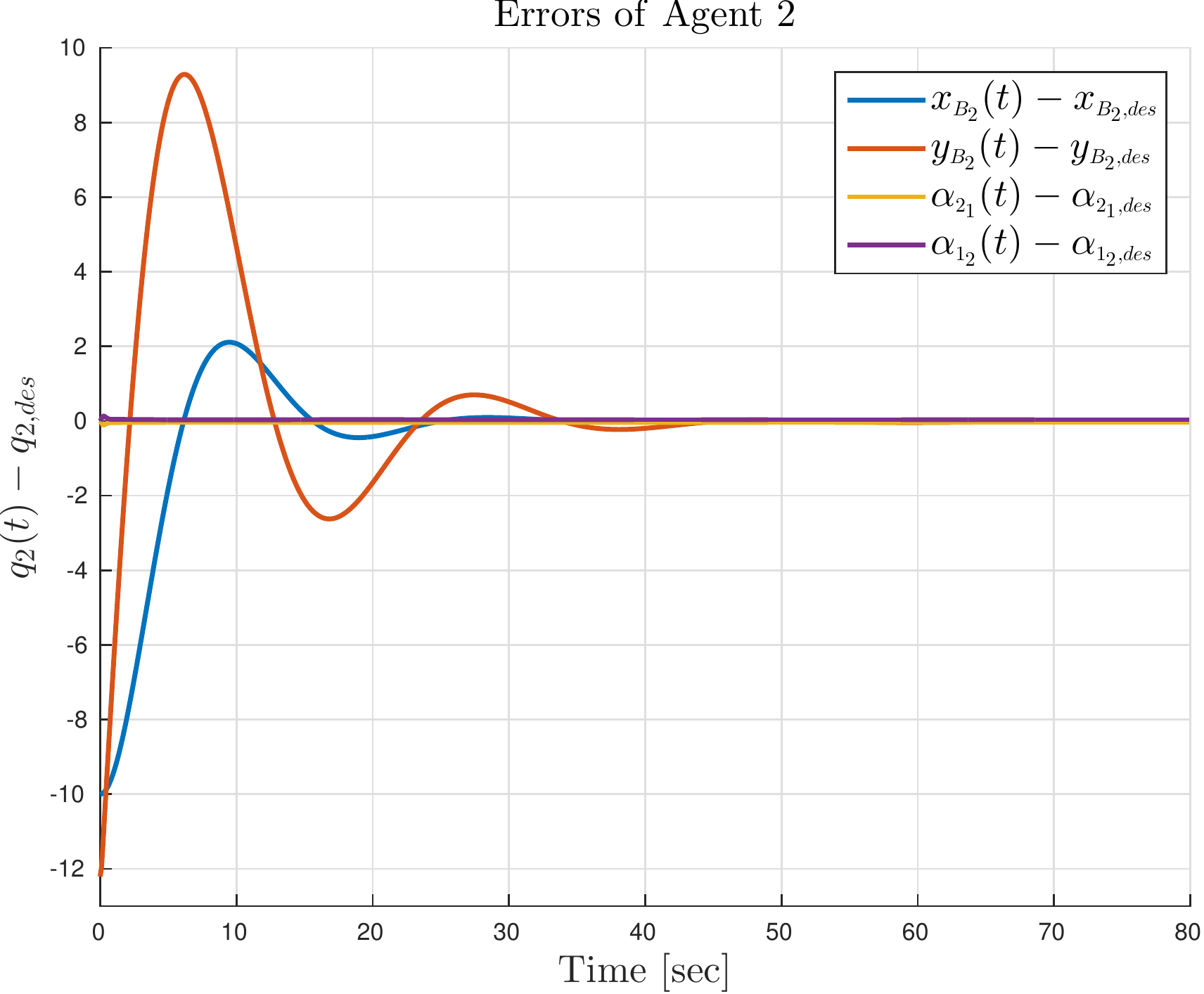}	
	\caption{The errors of vehicle $2$ as well as the errors of the manipulator.\label{fig:sim4}}
\end{figure} 

\begin{figure}[t!]
	\centering
	\includegraphics[width = 0.6\textwidth]{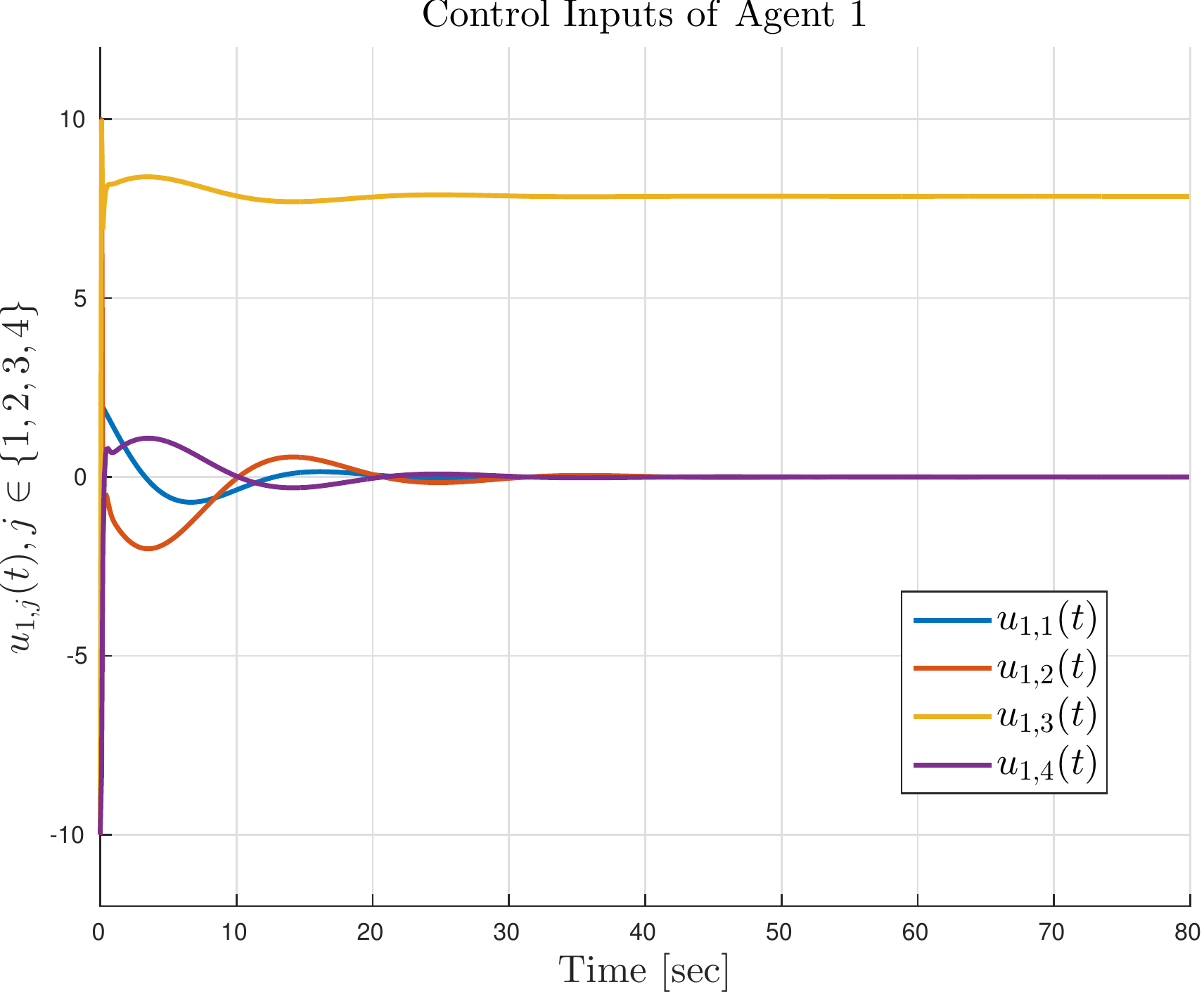}
	\caption{The control inputs of the actuators of agent $1$.\label{fig:sim7}}
\end{figure} 

\begin{figure}[t!]
	\centering
	\includegraphics[width = 0.6\textwidth]{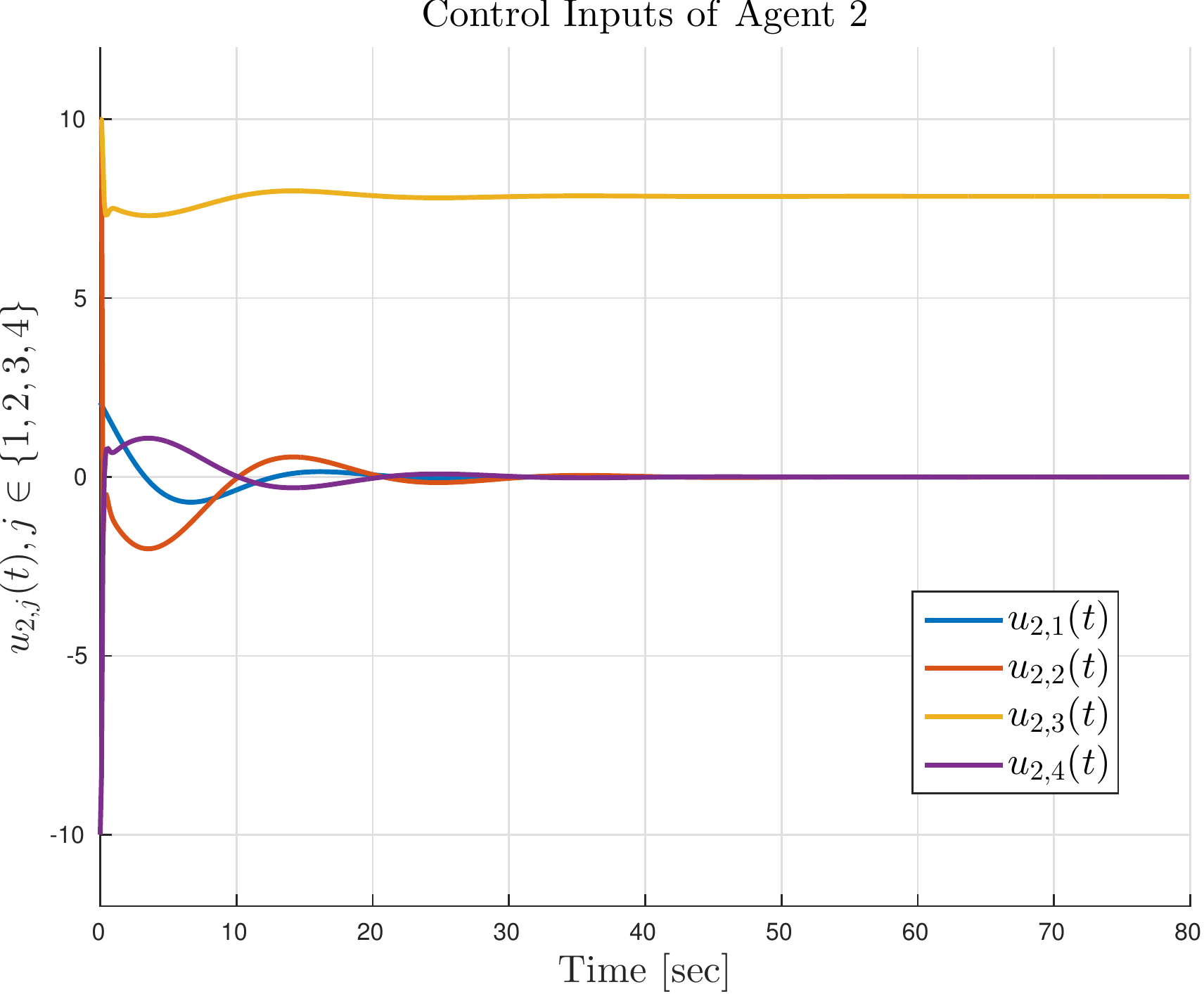}	
	\caption{The control inputs of the actuators of agent $2$.\label{fig:sim8}}
\end{figure} 

\textbf{Scenario 2:} Consider $N=3$ ground vehicles equipped with $2$ DOF manipulators, rigidly grasping an object with $n_1 = n_2 = n_3 = 4, n = n_1+n_2+n_3 = 12$. From \eqref{eq:main_system} we have that $x = [x_{\scriptscriptstyle O}^\top, v_{\scriptscriptstyle O}^\top, q^\top]^\top \in \mathbb{R}^{20}$, $u \in \mathbb{R}^{12}$, with $x_{\scriptscriptstyle O} = [p_{\scriptscriptstyle O}^\top, \phi_{\scriptscriptstyle O}]^\top \in \mathbb{R}^4$, $v_{\scriptscriptstyle O} = [\dot{p}_{\scriptscriptstyle O}^\top, \omega_{\scriptscriptstyle x_O}]^\top \in \mathbb{R}^4$, $p_{\scriptscriptstyle O} = [x_{\scriptscriptstyle O}, y_{\scriptscriptstyle O}, z_{\scriptscriptstyle O}]^\top \in \mathbb{R}^3$, $q = [q_1^\top, q_2^\top, q_3^\top]^\top \in \mathbb{R}^{12}$, $q_i = [p_{\scriptscriptstyle B_i}^\top, \alpha_i^\top]^\top \in \mathbb{R}^4$, $p_{\scriptscriptstyle B_i} = [x_{\scriptscriptstyle B_i}, y_{\scriptscriptstyle B_i}]^\top \in \mathbb{R}^2$, $\alpha_i = [\alpha_{i_1}, \alpha_{i_2}]^\top \in \mathbb{R}^2, i \in \{1,2\}$. The manipulators become singular when $\sin(\alpha_{i_1}) = 0 \}, i \in \{1,2,3\}$, thus the state constraints for the manipulators are set to:
\begin{align}
	\varepsilon < \alpha_{1_1} < \frac{\pi}{2}-\varepsilon, & -\frac{\pi}{2}+\varepsilon < \alpha_{1_2} < \frac{\pi}{2}-\varepsilon, \notag \\
	-\frac{\pi}{2} + \varepsilon < \alpha_{2_1} < -\varepsilon,& -\frac{\pi}{2}+\varepsilon < \alpha_{2_2} < \frac{\pi}{2}-\varepsilon. \notag
\end{align}
We also consider the input constraints:
\begin{equation*}
	-10 \le u_{i,j}(t) \le 10, i \in \{1, 2\}, j \in \{1,\dots,4\}.
\end{equation*}
The initial conditions are set to:
\begin{align*}
	x_{\scriptscriptstyle O}(0) &= \left[0, -2.2071, 0.9071, \frac{\pi}{2} \right]^\top, v_{\scriptscriptstyle O}(0) = \left[0, 0, 0, 0\right]^\top, \\
	q_{1}(0) &= \left[0.5, 0, \frac{\pi}{4}, \frac{\pi}{4}\right]^\top, q_{2}(0) = \left[0, -4.4142, -\frac{\pi}{4}, -\frac{\pi}{4}\right]^\top, \\
	q_{3}(0) &= \left[-0.5, 0, \frac{\pi}{4}, \frac{\pi}{4}\right]^\top.
\end{align*}
The desired goal states are set to: 
\begin{align*}
x_{\scriptscriptstyle O, \text{des}} &= \left[5, -2.2071, 0.9071, \frac{\pi}{2} \right]^\top, v_{\scriptscriptstyle O, \text{des}} = \left[0, 0, 0, 0\right]^\top, \\
q_{1, \text{des}} &= \left[5.5, 0, \frac{\pi}{4}, \frac{\pi}{4}\right]^\top, q_{2, \text{des}} = \left[5, -4.4142, -\frac{\pi}{4}, -\frac{\pi}{4}\right]^\top, \\
q_{3, \text{des}} &= \left[4.5, 0.0, \frac{\pi}{4}, \frac{\pi}{4}\right]^\top.
\end{align*}  
The sampling time is $h = 0.1 \sec$, the horizon is set to $T_p = 0.5 \sec$, and the total simulation time is $100 \sec$; The matrices $P, Q, R$ are set to:
\begin{equation}
	P = Q = 0.5  I_{20 \times 20}, R = 0.5 I_{12 \times 12}. \notag 
\end{equation}
The simulation results are depicted in Fig. \ref{fig:sim9}- Fig. \ref{fig:sim16},  which shows that the states of the agents as well as the states of the object converge to the desired ones while guaranteeing that all state and input constraints are met. The simulation scenarios were carried out by using the NMPC toolbox given in \cite{grune_2011_nonlinear_mpc} and they took $23500 \sec$, $45547 \sec$ for Scenario $1$ and Scenario $2$, respectively, in MATLAB Environment on a desktop with $8$ cores, $3.60$ GHz SPU and $16$GB of RAM.

\section{Conclusions and Future Work} \label{sec:conclusions}

In this work we proposed a NMPC scheme for the cooperative transportation of an object rigidly grasped by $N$ robotic agents. The proposed control scheme deals with singularities of the agents, inter-agent collision avoidance as well as collision avoidance between the agents and the object with the workspace obstacles. We proved the feasibility and  convergence analysis of the proposed methodology and simulation results verified the efficiency of the approach. Future efforts will be devoted towards including load sharing coefficients, internal force regulation, and complete decentralization of the proposed method. Finally, we will try to decrease the overall complexity and carry out real-time experiments.

\begin{figure}[t!]
	\centering
	\includegraphics[width = 0.6\textwidth]{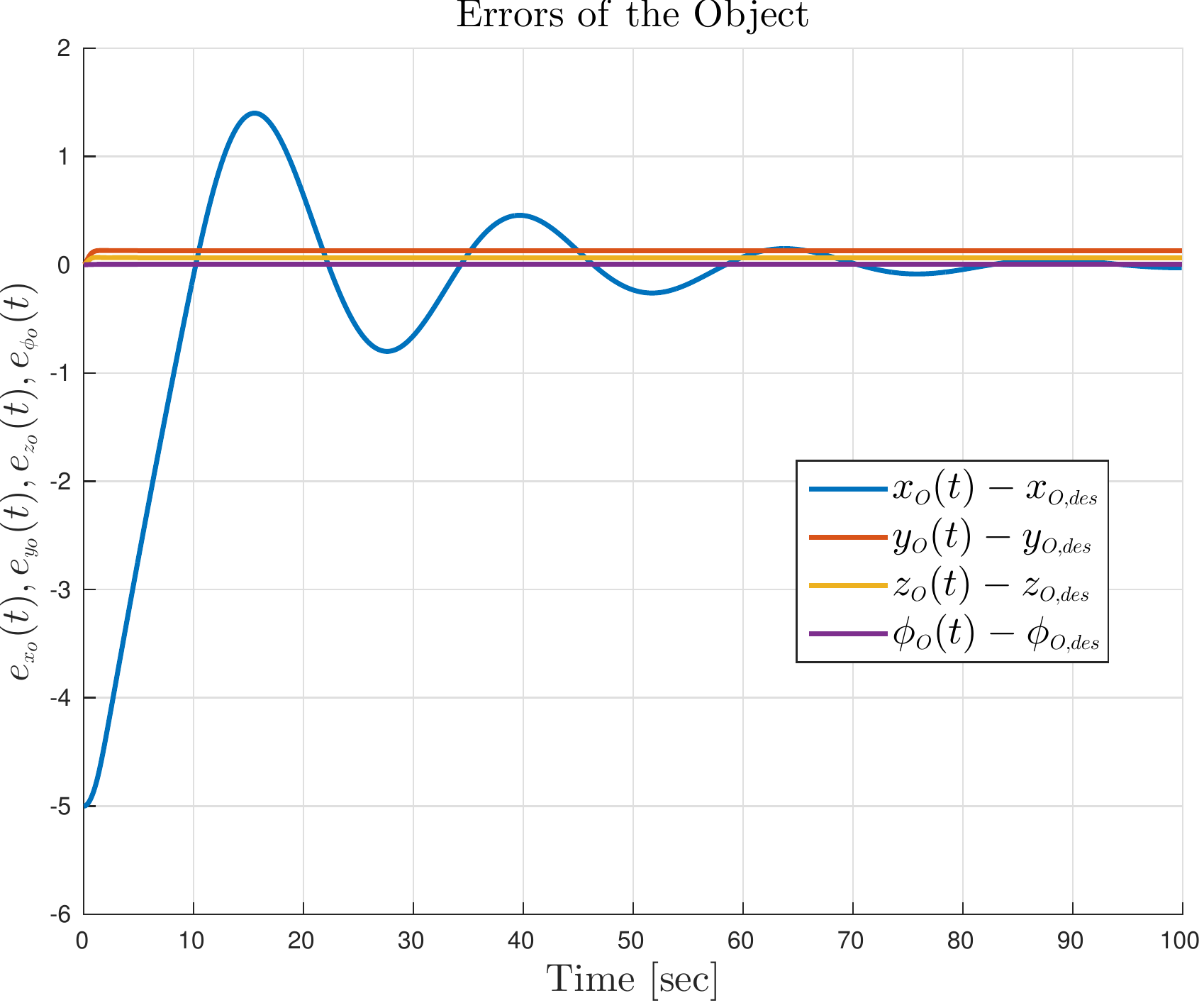}
	\caption{The errors of the object.\label{fig:sim9}}
\end{figure} 

\begin{figure}[t!]
	\centering
	\includegraphics[width = 0.6\textwidth]{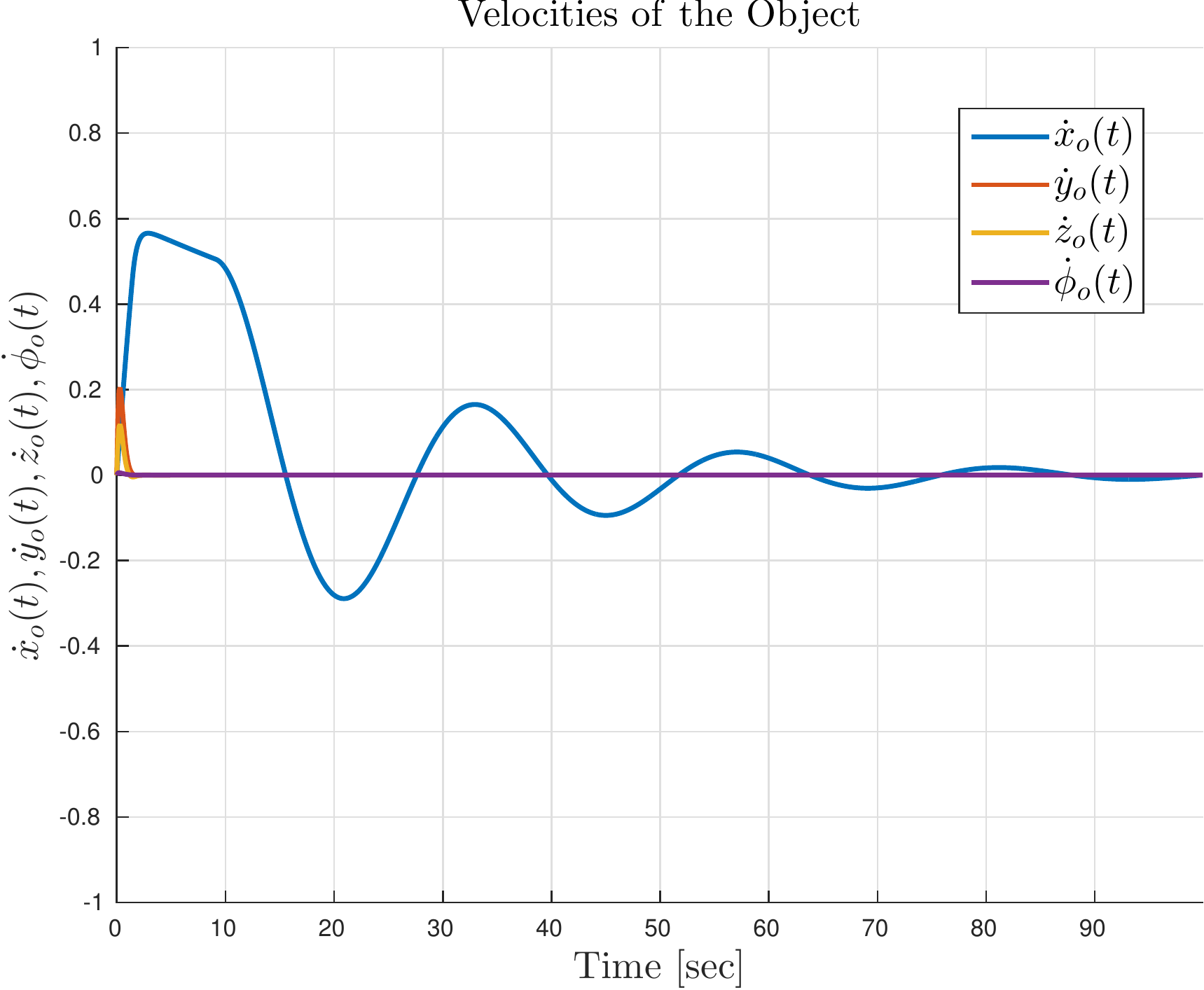}
	\caption{The velocities of the object.\label{fig:sim10}}
\end{figure} 

\begin{figure}[t!]
	\centering
	\includegraphics[width = 0.6\textwidth]{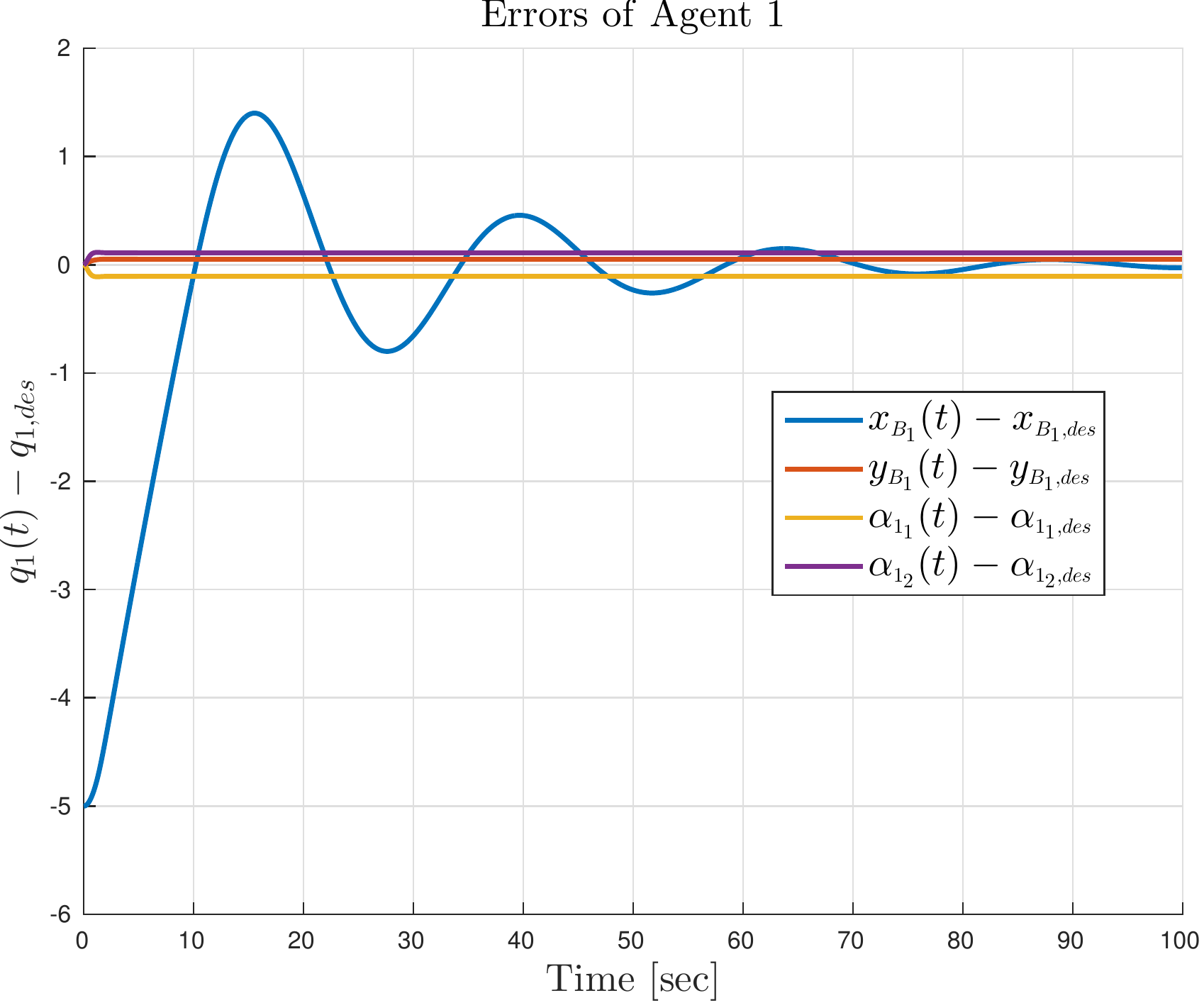}	
	\caption{The errors of vehicle $1$ as well as the errors of the manipulator.\label{fig:sim11}}
\end{figure} 

\begin{figure}[t!]
	\centering
	\includegraphics[width = 0.6\textwidth]{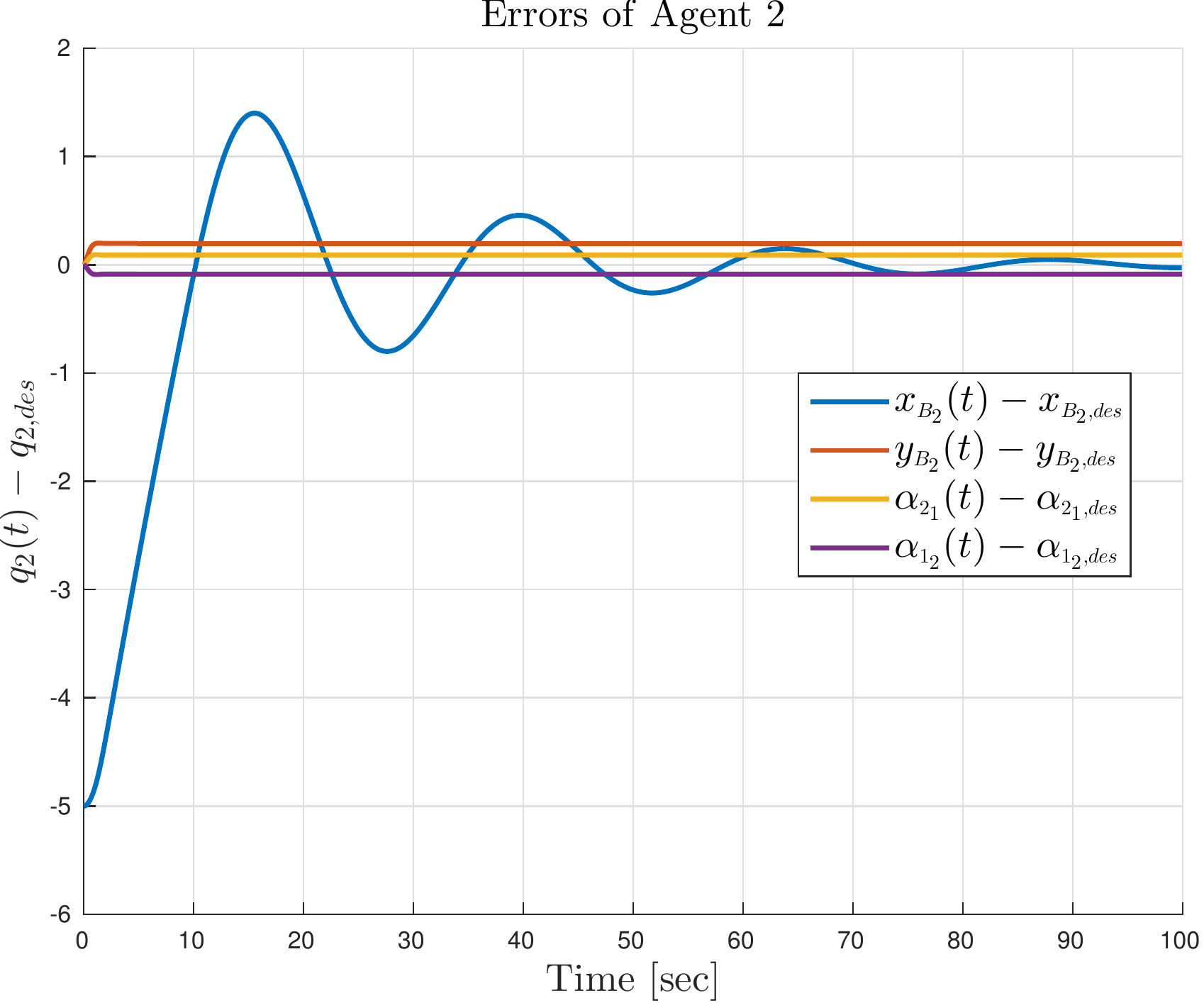}	
	\caption{The errors of vehicle $2$ as well as the errors of the manipulator.\label{fig:sim12}}
\end{figure} 

\begin{figure}[t!]
	\centering
	\includegraphics[width = 0.6\textwidth]{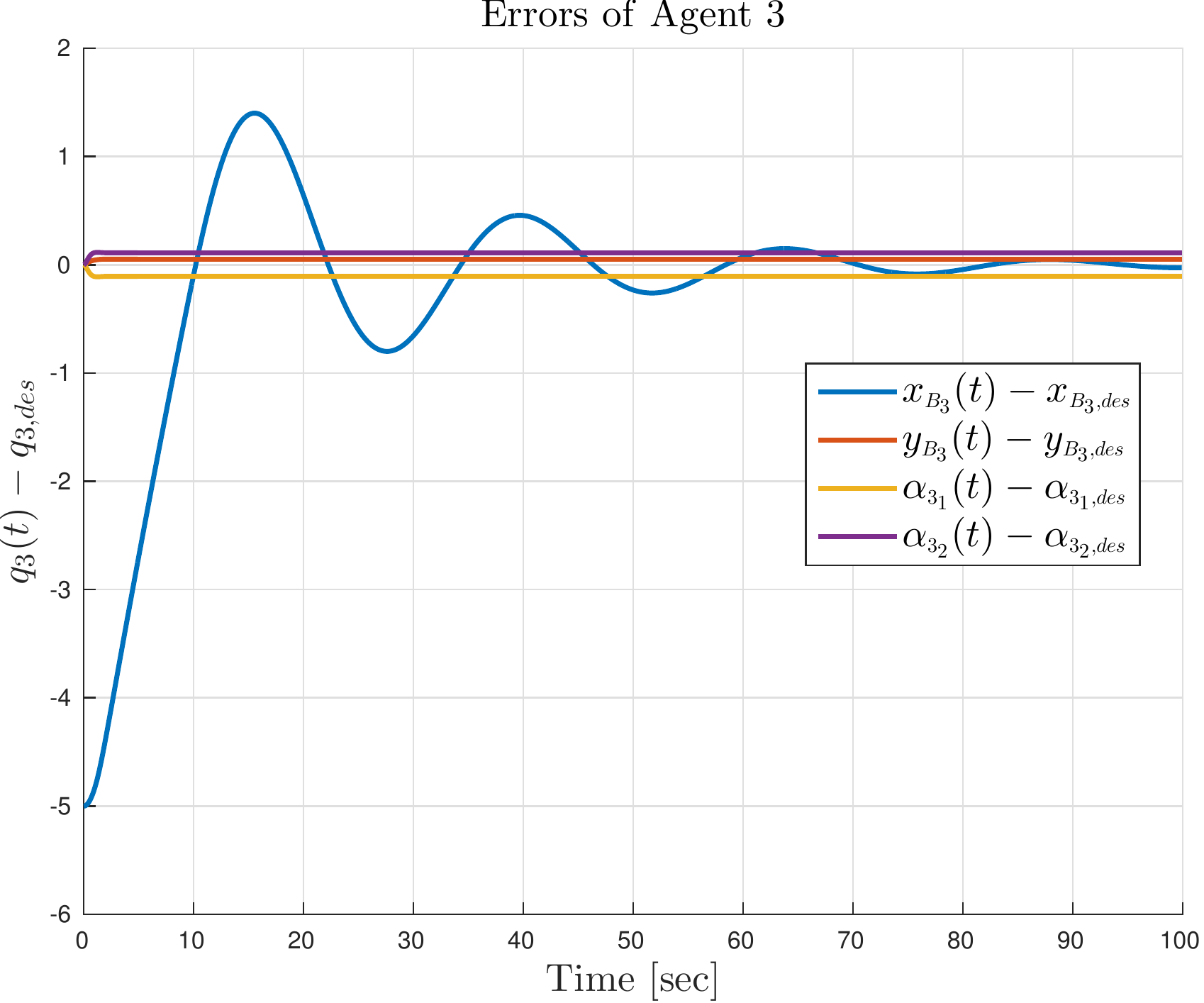}	
	\caption{The errors of vehicle $3$ as well as the errors of the manipulator.\label{fig:sim13}}
\end{figure} 

\begin{figure}[t!]
	\centering
	\includegraphics[width = 0.6\textwidth]{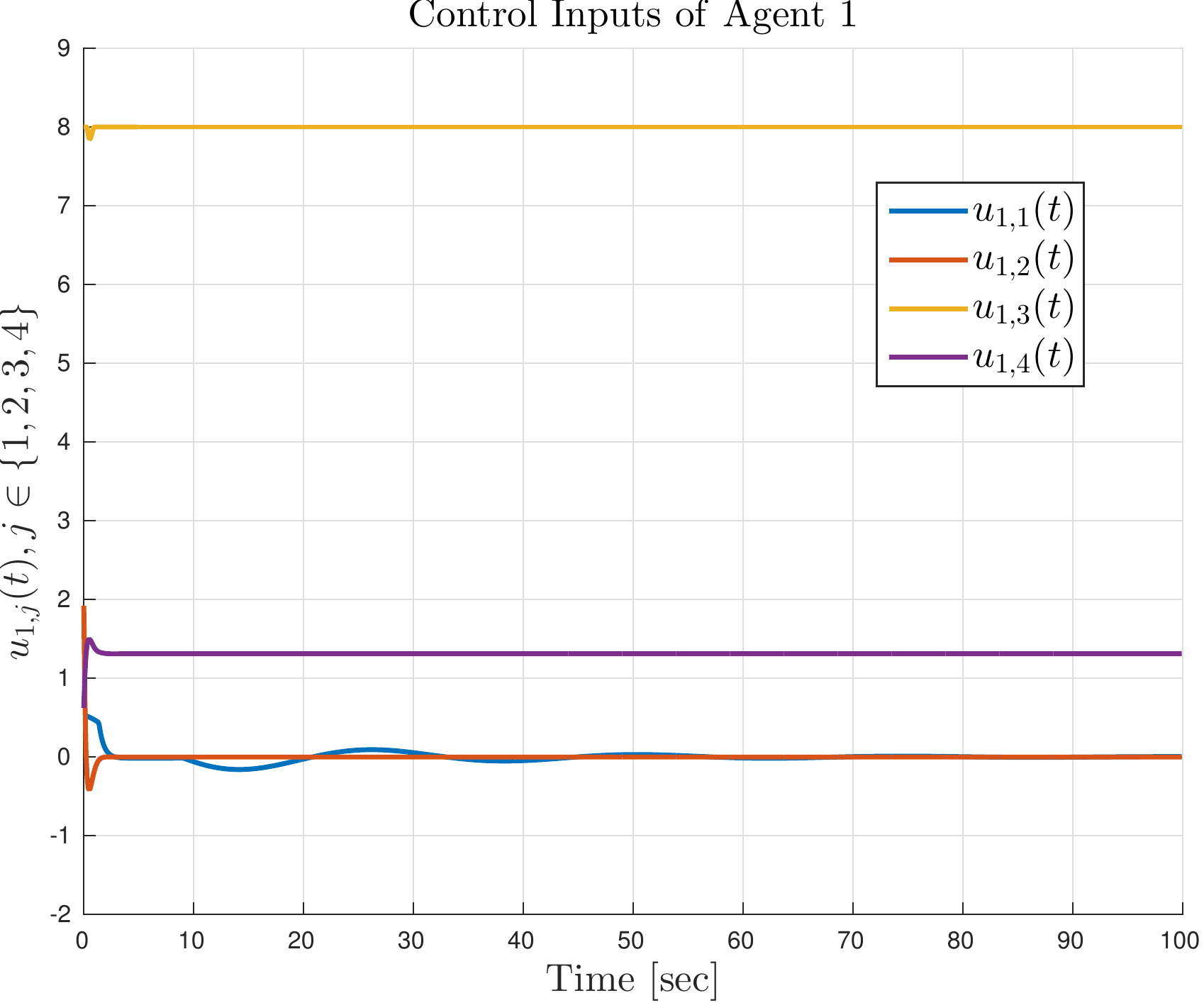}
	\caption{The control inputs of the actuators of agent $1$.\label{fig:sim14}}
\end{figure} 

\begin{figure}[t!]
	\centering
	\includegraphics[width = 0.6\textwidth]{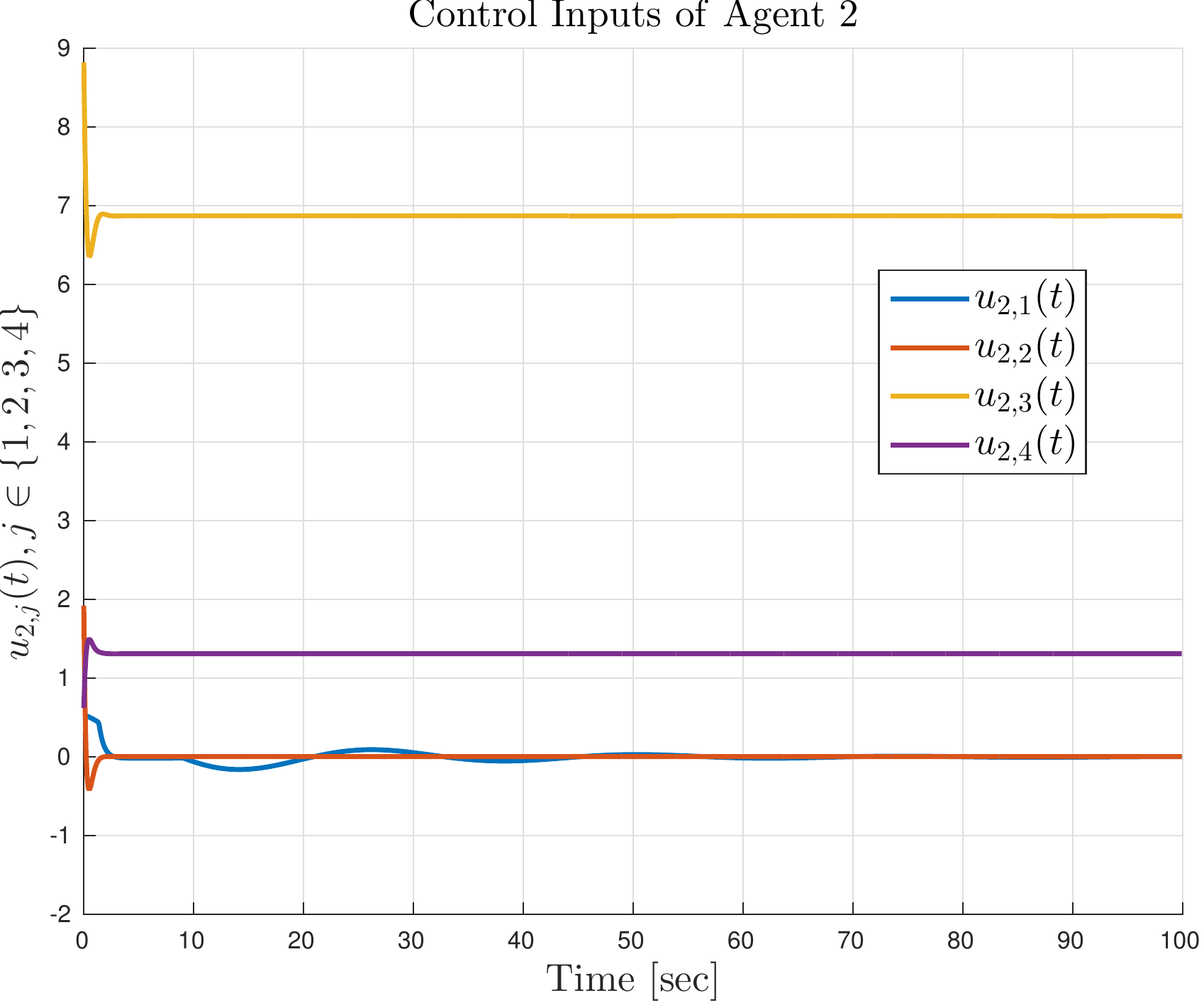}
	\caption{The control inputs of the actuators of agent $3$.\label{fig:sim15}}
\end{figure} 

\begin{figure}[t!]
	\centering
	\includegraphics[width = 0.6\textwidth]{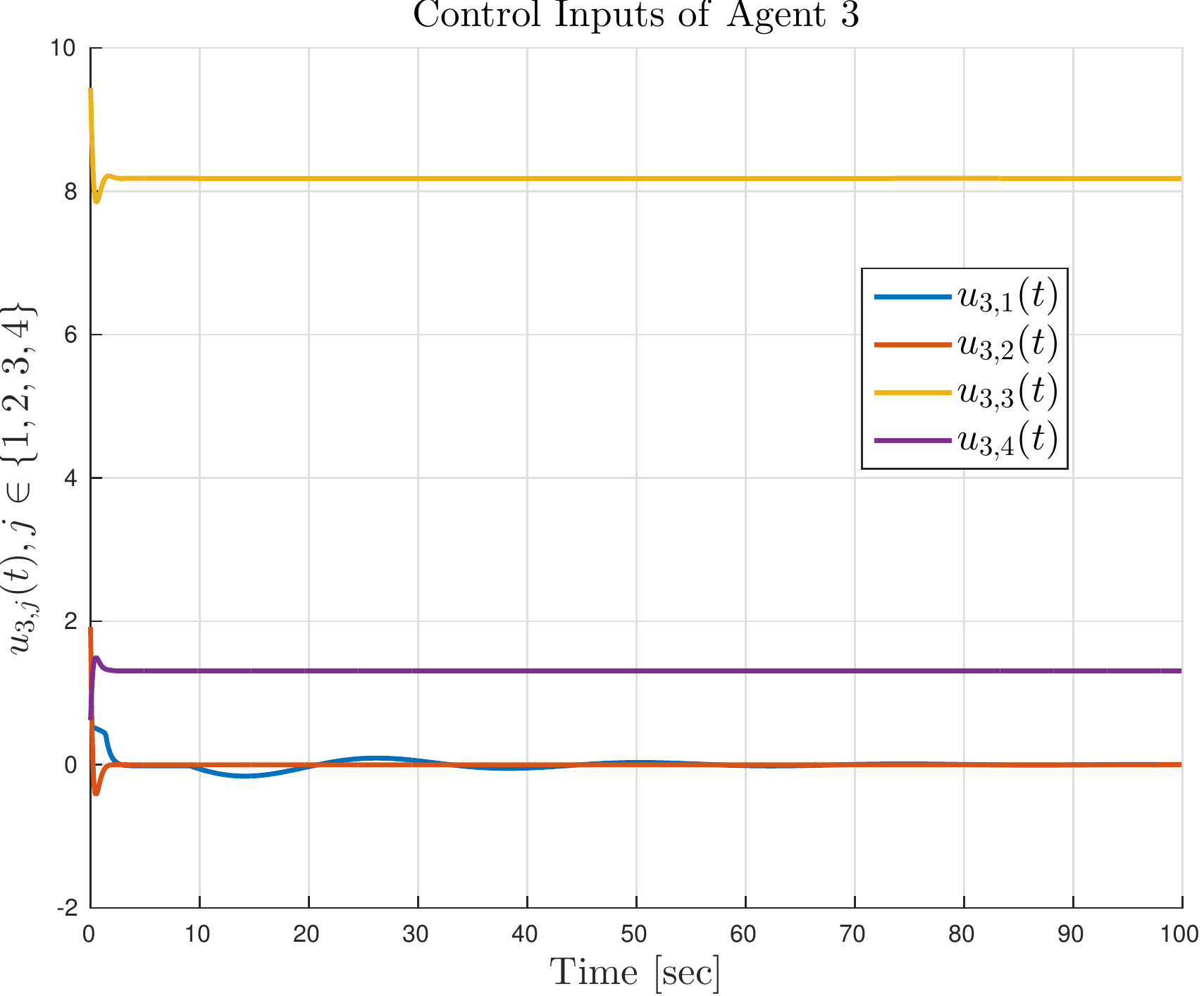}
	\caption{The control inputs of the actuators of agent $3$.\label{fig:sim16}}
\end{figure} 

\appendices

\section{Proof of Lemma \ref{lemma:F_i_bounded_K_class}} \label{app:proof_lemma_1}

\noindent By invoking the fact that:
\begin{align} \label{eq:rayleigh_inequality}
\lambda_{\min}(P) \|y\|^2 \le y^\top P y \le \lambda_{\max}(P) \|y\|^2, \forall y \in \mathbb{R}^n, P \in \mathbb{R}^{n \times n}, P = P^\top > 0,
\end{align}
we have:
\begin{align*}
e^\top Q e + u^\top R u & \le \lambda_{\max}(Q) \|e\|^2 + \lambda_{\max}(R) \|u\|^2  = \max \{\lambda_{\max}(Q), \lambda_{\max}(R) \} \|z\|^2, \notag
\end{align*}
and:
\begin{align*}
e^\top Q e + u^\top R u & \ge \lambda_{\min}(Q) \|e\|^2 + \lambda_{\min}(R) \|u\|^2 \notag \\
& = \min \{\lambda_{\min}(Q), \lambda_{\min}(R) \} \|z\|^2, \notag
\end{align*}
where $z = \left[ e^\top, u^\top\right]^\top$. Thus, we get:
\begin{align*}
\min \{\lambda_{\min}(Q), \lambda_{\min}(R) \} \|z\|^2 & \le e^\top Q e + u^\top R u \le \max \{\lambda_{\max}(Q), \lambda_{\max}(R) \} \|z\|^2.
\end{align*}
By defining the $\mathcal{K}_{\infty}$ functions $\alpha_1$, $\alpha_2 : \mathbb{R}_{\ge 0}  \to \mathbb{R}_{\ge 0}$:
\begin{align*}
\alpha_1(y) \triangleq m \|y \|^2, \alpha_2(y) \triangleq \max \{\lambda_{\max}(Q), \lambda_{\max}(R) \} \|y \|^2,
\end{align*}
and the parameter $m \in \mathbb{R}_{>0}$ by:
\begin{align} \label{eq:mi}
m \triangleq \min \{\lambda_{\min}(Q), \lambda_{\min}(R) \},
\end{align}
we get:
\begin{align} \label{eq:F_lower_bound}
\alpha_1\big(\|z\|\big) \leq F \big(e, u\big) \leq \alpha_2\big(\|z\|\big),
\end{align}
which leads to the conclusion of the proof. \qed

\section{Proof of Lemma 2} \label{app:lemma_2}

\begin{proof}
	\noindent For every $e(t) \in \mathcal{E}_f$, the following holds:
	\begin{align}
	|V(e_1)-V(e_2) | & = |e_1^\top P e_1 - e_2^\top P e_2 | \notag \\
	& = |e_1^\top P e_1 +e_1^\top P e_2 -e_1^\top P e_2 - e_2^\top P e_2 | \notag \\
	& = |e_1^\top P (e_1-e_2) - e_2^\top P ( e_1 - e_2) | \notag \\
	& \le |e_1^\top P (e_1-e_2)| + |e_2^\top P ( e_1 - e_2) |. \label{eq:lemma_1_proof_step_1}
	\end{align}
	By employing the property that:
	\begin{equation*}
	|x^\top A y| \le \sigma_{\max}(A) \|x\| \|y\|, \forall \ x,y \in \mathbb{R}^n, A \in \mathbb{R}^{n \times n},
	\end{equation*}
	\eqref{eq:lemma_1_proof_step_1} is written as:
	\begin{align*}
	|V(e_1)-V(e_2) | &\le \sigma_{\max}(P) \|e_1\| \|e_1-e_2\| +\sigma_{\max}(P) \|e_2\| \|e_1-e_2\| \notag \\
	&= \sigma_{\max}(P) (\|e_1\| + \|e_2\|) \|e_1-e_2\| \notag \\
	&\le \sigma_{\max}(P) (\varepsilon_0 + \varepsilon_0) \|e_1-e_2\| \notag \\ 
	&= \left[ 2 \varepsilon_0 \sigma_{\max}(P) \right] \|e_1-e_2\|.
	\end{align*}
	which completes the proof.
\end{proof}

\section{Feasibility Analysis} \label{app:feasibility}

	Consider any sampling time instant $t_i$ for which a solution exists. In between $t_i$ and $t_{i+1}$, the optimal control input $\hat{u}^\star (s; e(t_i)), s \in [t_i, t_{i+1})$ is implemented. According to \eqref{eq:predicted_state_relation}, it holds that:
	\begin{equation*}
	e(t_{i+1}) = \hat{e}(t_{i+1}; \hat{u}^\star (\cdot; e(t_i)), e(t_i)).
	\end{equation*}
	The remaining piece of the optimal control input $\hat{u}^\star (s; e(t_i)), s \in [t_{i+1}, t_i+T_p]$ satisfies the state and input constraints $E, U$, respectively. Furthermore, 
	\begin{equation*}
	\hat{e}(t_{i}+T_p; \hat{u}^\star (\cdot; e(t_i)), e(t_i)) \in \mathcal{E}_f,
	\end{equation*}
	and we know from Assumption 2b of Theorem 1 that for all $e(t) \in \mathcal{E}_f$, there exists at least one control input $u_f(\cdot)$ that renders the set $\mathcal{E}_f$ invariant over $h$. Picking any such input, a feasible control input $\bar{u}(\cdot; e(t_{i+1}))$, at time instant $t_{i+1}$, may be the following:
	\begin{align} \label{eq:u_bar_feas}
	&\bar{u}(s; e(t_{i+1})) = \begin{cases}
	\hat{u}^\star (s; e(t_i)), \qquad & s \in [t_{i+1}, t_i+T_p], \\
	u_f (\hat{e}(t_i+T_p; u^\star(\cdot), e(t_i))), & s \in [t_{i}+T_p, t_{i+1}+T_p]. 
	\end{cases}
	\end{align}
	Thus, from feasibility of $\hat{u}^\star(s, e(t_i))$ and the fact that $u_f(e(t)) \in U$, for all $e(t) \in \mathcal{E}_f$, it follows that:
	\begin{equation*}
	\bar{u}(s; e(t_{i+1})) \in U, \forall \ s \in [t_{i+1}, t_i+T_p].
	\end{equation*}
	Hence, the feasibility at time $t_i$ implies feasibility at time $t_{i+1}$. Therefore, if the OCP \eqref{eq:mpc_minimazation} - \eqref{eq:mpc_terminal_set} is feasible at time $t = 0$, it remains feasible for every $t \ge 0$.

\section{Convergence Analysis} \label{app:convergence_analysis}

The second part involves proving convergence of the state $e$ in the terminal set $\mathcal{E}_f$. In order to prove this, it must be shown that a proper value function is decreasing along the solution trajectories starting at a sampling time $t_i$. Consider the optimal value function $J^\star(e(t_i))$, as is given in \eqref{eq:J_star}. Consider also the cost of the feasible control input, indicated by:
\begin{equation} \label{eq:J_bar}
\bar{J}(e(t_{i+1})) \triangleq \bar{J}(e(t_{i+1}), \bar{u}(\cdot; e(t_{i+1}))),
\end{equation}
where $t_{i+1} = t_i + h$, as is given in \eqref{eq:t_i_equals_h}. Define:
\begin{align}
u_1(s) &= \bar{u}(s; e(t_{i+1})), \label{eq:u_1} \\
e_1(s)  &= \bar{e}(s; u_1(s), e(t_{i+1})), s > t_{i+1}. \notag
\end{align}
$e_1(s)$ stands for the predicted state $e$ at time $s$, based on the measurement of the state $e$ at time $t_{i+1}$, while using the feasible control input $\bar{u}(s; e(t_{i+1}))$. Let us also define th the following terms:
\begin{align}
u_2(s) &= \hat{u}^\star(s; e(t_{i})), \label{eq:u_2}\\
e_2(s)  &= \hat{e}(s; u_2(s), e(t_{i})), s > t_{i+1}. \notag
\end{align}
\eqref{eq:u_1}, \eqref{eq:u_2} form convenient notations for the readability of the proof hereafter.

By employing \eqref{eq:mpc_minimazation}, \eqref{eq:J_star} and \eqref{eq:J_bar}, the difference between the optimal and feasible cost is given by:
\begin{align}
&\bar{J}(e(t_{i+1})) - J^\star(e(t_i)) = V(e_1(t_{i+1}+T_p)) + \int_{t_{i+1}}^{t_{i+1}+T_p} \Big[ F(e_1(s), u_1(s)) \Big] ds \notag \\
&\hspace{8mm}-V(e_2(t_{i}+T_p)) - \int_{t_{i}}^{t_{i}+T_p} \Big[ F(e_2(s), u_2(s)) \Big] ds \notag \\
&=V(e_1(t_{i+1}+T_p)) + \int_{t_{i+1}}^{t_{i}+T_p} \Big[ F(e_1(s), u_1(s)) \Big] ds \notag \\ 
&+\int_{t_{i}+T_p}^{t_{i+1}+T_p} \Big[ F(e_1(s), u_1(s)) \Big] ds -V(e_2(t_{i}+T_p)) \notag \\
&\hspace{15mm}- \int_{t_{i}}^{t_{i+1}} \Big[ F(e_2(s), u_2(s)) \Big] ds -\int_{t_{i+1}}^{t_{i}+T_p} \Big[ F(e_2(s), u_2(s)) \Big] ds. \label{eq:lyap1}
\end{align}
Note that, from \eqref{eq:u_bar_feas}, the following holds:
\begin{equation} \label{eq:verify_u_bar}
\bar{u}(s; e(t_{i+1})) = \hat{u}^\star(s; e(t_i)), \forall \ s \in [t_{i+1}, t_i+T_p].
\end{equation}
By combining \eqref{eq:u_1}, \eqref{eq:u_2} and \eqref{eq:verify_u_bar}, it yields that:
\begin{equation} \label{eq:u_1_equals_u_2}
u_1(s) = u_2(s) = \bar{u}(s), \forall \ s \in [t_{i+1}, t_i+T_p],
\end{equation}
which implies that:
\begin{equation} \label{eq:e_1_equals_e_2}
e_1(s) = e_2(s), \forall \ s \in [t_{i+1}, t_i+T_p].
\end{equation}
The combination of \eqref{eq:u_1_equals_u_2} and \eqref{eq:e_1_equals_e_2} implies that:
\begin{equation*}
F(e_1(s), u_1(s)) = F(e_1(s), u_1(s)), \forall \ s \in [t_{i+1}, t_i+T_p].
\end{equation*}
which implies that:
\begin{align} 
&\int_{t_{i+1}}^{t_{i}+T_p} \Big[ F(e_1(s), u_1(s)) \Big] ds = \int_{t_{i+1}}^{t_{i}+T_p} \Big[ F(e_2(s), u_2(s)) \Big] ds. \label{eq:F_1_equals_F_2}
\end{align}
By employing \eqref{eq:F_1_equals_F_2}, \eqref{eq:lyap1} becomes:
\begin{align}
\bar{J}(e(t_{i+1})) - J^\star(e(t_i)) = \ & V(e_1(t_{i+1}+T_p))  + \int_{t_{i}+T_p}^{t_{i+1}+T_p} \Big[ F(e_1(s), u_1(s)) \Big] ds \notag \\
& -V(e_2(t_{i}+T_p)) - \int_{t_{i}}^{t_{i+1}} \Big[ F(e_2(s), u_2(s)) \Big] ds. \label{eq:lyap2}
\end{align}
Due to the fact that $t_{i+1}+T_p-(t_i+T_p) = t_{i+1}-t_i= h$, and the Assumption 2b of Theorem 1 holds for one sampling period $h$, by integrating this inequality from $t_i+T_p$ to $t_{i+1}+T_p$ and we get the following:
\begin{align}
&\hspace{-3mm} \int_{t_{i}+T_p}^{t_{i+1}+T_p} \Big[ \frac{\partial V}{\partial{e}} f_e(e_1(s), u_1(s)) + F(e_1(s), u_1(s)) \Big] ds \le 0 \notag \\ 
\Leftrightarrow & \int_{t_{i}+T_p}^{t_{i+1}+T_p} \Big[ \dot{V}(e_1(s)) \Big] ds  +\int_{t_{i}+T_p}^{t_{i+1}+T_p} \Big[F(e_1(s), u_1(s)) \Big] ds \le 0 \notag \\
\Leftrightarrow & V(e_1(t_{i+1}+T_p)) -V(e_1(t_{i}+T_p))  + \int_{t_{i}+T_p}^{t_{i+1}+T_p} \Big[F(e_1(s), u_1(s)) \Big] ds \le 0 \notag \\
\Leftrightarrow & V(e_1(t_{i+1}+T_p)) -V(e_1(t_{i}+T_p)) + \int_{t_{i}+T_p}^{t_{i+1}+T_p} \Big[F(e_1(s), u_1(s)) \Big] ds \le \notag \\ 
&\hspace{25mm} V(e_2(t_{i}+T_p)) - V(e_2(t_{i}+T_p)) \notag \\
\Leftrightarrow & V(e_1(t_{i+1}+T_p))  +  \int_{t_{i}+T_p}^{t_{i+1}+T_p} \Big[F(e_1(s), u_1(s)) \Big] ds -V(e_2(t_{i}+T_p)) \le \notag \\ 
&\hspace{22mm} V(e_1(t_{i}+T_p)) - V(e_2(t_{i}+T_p)). \notag 
\end{align}
By employing the property $y \le |y|, \forall y \in \mathbb{R}$, we get:
\begin{align}
& V(e_1(t_{i+1}+T_p))  + \int_{t_{i}+T_p}^{t_{i+1}+T_p} \Big[F(e_1(s), u_1(s)) \Big] ds -V(e_2(t_{i}+T_p)) \le \notag \\ 
&\hspace{15mm} \left| V(e_1(t_{i}+T_p)) - V(e_2(t_{i}+T_p)) \right|. \label{eq:V_e_1-V_e_2}
\end{align}
By employing Lemma 2, we have that:
\begin{align}
&\left| V(e_1(t_{i}+T_p)) - V(e_2(t_{i}+T_p)) \right| \le  L_V \|e_1(t_{i}+T_p) - e_2(t_{i}+T_p) \|. \label{eq:lip_e_1_e_2_t_i}
\end{align}
By combining \eqref{eq:V_e_1-V_e_2} and \eqref{eq:lip_e_1_e_2_t_i} we get:
\begin{align}
& V(e_1(t_{i+1}+T_p)) + \int_{t_{i}+T_p}^{t_{i+1}+T_p} \Big[F(e_1(s), u_1(s)) \Big] ds -V(e_2(t_{i}+T_p))  \le \notag \\ 
&\hspace{25mm} L_V \|e_1(t_{i}+T_p) - e_2(t_{i}+T_p) \| \label{eq:norm_V_e_1-V_e_2}
\end{align}
For $s = t_i+T_p$, \eqref{eq:e_1_equals_e_2} gives:
\begin{equation} \label{eq:e_1_equals_e_2_t_p}
e_1(t_i+T_p) = e_2(t_i+T_p).
\end{equation}
By combining \eqref{eq:e_1_equals_e_2_t_p} and \eqref{eq:norm_V_e_1-V_e_2} we have:
\begin{align}
&\hspace{-2.3mm}V(e_1(t_{i+1}+T_p))  + \notag \\ 
&\hspace{-2.3mm}\int_{t_{i}+T_p}^{t_{i+1}+T_p} \Big[F(e_1(s), u_1(s)) \Big] ds -V(e_2(t_{i}+T_p)) \le 0. \label{eq:norm_V_e_1-V_e_2_ineq_zero}
\end{align}
By combining \eqref{eq:lyap2} with \eqref{eq:norm_V_e_1-V_e_2_ineq_zero}, the following holds:
\begin{align}
&\bar{J}(e(t_{i+1})) - J^\star(e(t_i)) \le - \int_{t_{i}}^{t_{i+1}} \Big[ F(e_2(s), u_2(s)) \Big] ds. \label{eq:lyap3}
\end{align}
By substituting $e = e_2(s), u = u_2(s)$ in \eqref{eq:F_lower_bound} we get:
\begin{equation*}
F(e_2(s),u_2(s)) \ge m \|z_2(s)\|^2,
\end{equation*}
where $z_2\triangleq [e_2, u_2]^\top$. The latter is equivalent to:
\begin{align}
&\int_{t_{i}}^{t_{i+1}} \Big[ F(e_2(s),u_2(s)) \Big] ds \ge m \int_{t_{i}}^{t_{i+1}} \|z_2(s)\|^2  ds \notag \\
\Leftrightarrow &-\int_{t_{i}}^{t_{i+1}} \Big[ F(e_2(s),u_2(s)) \Big] ds \le -m \int_{t_{i}}^{t_{i+1}} \|z_2(s)\|^2  ds. \label{eq:int_f_e_2}
\end{align}
By combining \eqref{eq:lyap3} and \eqref{eq:int_f_e_2} we finally get:
\begin{align}
&\bar{J}(e(t_{i+1})) - J^\star(e(t_i)) \le -m \int_{t_{i}}^{t_{i+1}} \|z_2(s)\|^2  ds. \label{eq:lyap4}
\end{align}
It is clear that the optimal solution at time $t_{i+1}$ i.e., $J^\star(e(t_{i+1}))$  will not be worse than the feasible one at the same time i.e. $\bar{J}(e(t_{i+1}))$. Therefore, \eqref{eq:lyap4} implies:
\begin{equation}
J^\star(e(t_{i+1})) - J^\star(e(t_i)) \le -m \int_{t_{i}}^{t_{i+1}} \|z_2(s)\|^2s \le 0, \label{eq:lyap5}
\end{equation}
or, by using the fact that $\displaystyle \int_{t_0}^{t_{i}} \|z_2(s)\|^2ds$ $= \displaystyle \sum_{j = 0}^{i-1} \int_{t_j}^{t_{j+1}} \|z_2(s)\|^2ds$, equivalently, we obtain:
\begin{align}
& J^\star(e(t_{i+1})) - J^\star(e(t_i)) \le -m \int_{t_0}^{t_{i+1}} \|z_2(s)\|^2ds+ m \sum_{j = 0}^{i-1} \int_{t_j}^{t_{j+1}} \|z_2(s)\|^2ds. \label{eq:lyap6}
\end{align}
By using induction and the fact that $t_i = h \cdot i, t_{i+1} = h \cdot (i+1), \forall i \ge 0$, from \eqref{eq:t_i_equals_h}, \eqref{eq:lyap6} is written as:
\begin{equation}
J^\star(e(t_{i})) - J^\star(e(t_0)) \le -m \int_{t_0}^{t_{i}} \|z_2(s)\|^2ds. \label{eq:lyap7}
\end{equation}
Since $t_0 = 0$ we obtain:
\begin{equation}
J^\star(e(t_{i})) \le J^\star(e(0)) -m \int_{0}^{t_{i}} \|z_2(s)\|^2ds. \label{eq:lyap8}
\end{equation}
which implies that:
\begin{equation}
J^\star(e(t_{i})) \le J^\star(e(0)). \label{eq:lyap9}
\end{equation}
By combining \eqref{eq:lyap5}, \eqref{eq:lyap9}, we obtain:
\begin{equation}
J^\star(e(t_{i+1})) \le J^\star(e(t_{i})) \le  J^\star(e(0)), \forall \ t_i = i\cdot h, i \ge 0. \label{eq:lyap10}
\end{equation}
Therefore, the value function $J^\star(e(t_i))$ has proven to be non-increasing for all the sampling times. Let us define the function:
\begin{equation} \label{eq:V_function}
V(e(t)) = J^\star (e(s)) \le J^\star(e(0)), t \in\mathbb{R}_{\ge 0},
\end{equation}
where $s = \max \{ t_i : t_i \le t\}$. Since $J^\star(e(0))$ is bounded, \eqref{eq:V_function} implies that $V(e(t))$ is bounded. Since the signals $e(t), u(t)$ are bounded ($e(t) \in E, u(t) \in U$), according to \eqref{eq:error_dynamics}, it holds that $\dot{e}(t)$ is also bounded. From \eqref{eq:lyap8} we have that:
\begin{equation*}
V(e(t)) = J^\star (e(s)) \le J^\star(e(0)) -m \int_{0}^{s} \|z_2(s)\|^2ds.
\end{equation*}
which due to the fact that $s \le t$, is equivalent to:
\begin{equation} \label{eq:lyap11}
V(e(t)) \le J^\star(e(0)) -m \int_{0}^{t} \|z_2(s)\|^2ds, t \in\mathbb{R}_{\ge 0}.
\end{equation}
From \eqref{eq:lyap11}, we get:
\begin{equation} \label{eq:lyap12}
\int_{0}^{t} \|z_2(s)\|^2ds  \le \frac{1}{m} \left[ J^\star(e(0)) - V(e(t)) \right], t \in\mathbb{R}_{\ge 0}.
\end{equation}
Since $J^\star(e(0)), V(e(t))$ has been proven to be bounded, the term $\displaystyle \int_{0}^{t} \|z_2(s)\|^2ds$ is also bounded. Therefore, by employing Lemma 1, we have that $\|z_2(t)\|  \to 0$, as $t \to \infty$. The latter implies that:
\begin{equation*}
\lim_{t \to \infty} \|e(t)\| = 0 \Rightarrow e(t) \in \mathcal{E}_f, \ \text{as} \ t \to \infty,
\end{equation*}
and leads to the conclusion of the proof.

\bibliography{references}
\bibliographystyle{ieeetr}
\end{document}